\documentclass{article}

\usepackage{microtype}
\usepackage{graphicx}
\usepackage{tikz}
\usepackage{wrapfig}
\usepackage{subfigure}
\usepackage{booktabs} 
\usepackage{algorithm}
\usepackage{algorithmic}
\usepackage{comment}

\usepackage{mathtools}
\mathtoolsset{centercolon}

\newcommand{\nc}[1]{\newcommand{#1}}
\usepackage{xcolor}
\nc\comm[1]{\textcolor{ForestGreen}{\texttt{/*~#1~*/}}}
\nc{\knds}{\kern-\nulldelimiterspace}
\definecolor{verylightgray}{rgb}{0.8,0.8,0.8}
\definecolor{verylightgray}{rgb}{0.8,0.8,0.8}
\definecolor{brown}{rgb}{0.55, 0.25, 0.0}
\definecolor{forestgreen}{rgb}{0.0, 0.5, 0.0}
\definecolor{blue-violet}{rgb}{0.54, 0.17, 0.89}
\definecolor{dartmouthgreen}{rgb}{0.05, 0.5, 0.0}
\definecolor{dark-red}{rgb}{0.8, 0.0, 0.0}
\definecolor{light-blue}{rgb}{0.5, 0.5, 0.99}
\definecolor{brinkpink}{rgb}{0.85, 0.25, 0.5}
\definecolor{columbiablue}{rgb}{0.61, 0.87, 1.0}
\definecolor{cyan(process)}{rgb}{0.0, 0.55, 0.85}
\definecolor{darkcyan}{rgb}{0.0, 0.0, 0.8}
\definecolor{darkorange}{rgb}{0.85, 0.45, 0.0}
\definecolor{deeplilac}{rgb}{0.6, 0.33, 0.73}
\definecolor{electricultramarine}{rgb}{0.25, 0.0, 1.0}
\definecolor{electricviolet}{rgb}{0.56, 0.0, 1.0}

\newif\ifhidecomments
\hidecommentstrue
\ifhidecomments
    \newcommand{\AM}[1]{}
    \newcommand{\MJH}[1]{}
    \newcommand{\SP}[1]{}
    \newcommand{\JR}[1]{}
    \newcommand{\AR}[1]{}
    \newcommand{\SSS}[1]{}
    \newcommand{\MT}[1]{}
    \newcommand{\FV}[1]{}
\else
    \newcommand{\AM}[1]{\textcolor{cyan(process)}{{\bf AM:}{\small \texttt{~#1~}}}}
    \newcommand{\MJH}[1]{\textcolor{darkcyan}{{\bf MJH:}{\small \texttt{~#1~}}}}
    \newcommand{\SP}[1]{\textcolor{dark-red}{{\bf SP:}{\small \texttt{~#1~}}}}
    \newcommand{\JR}[1]{\textcolor{forestgreen}{{\bf JR:}{\small \texttt{~#1~}}}}
    \newcommand{\AR}[1]{\textcolor{brinkpink}{{\bf AR:}{\small \texttt{~#1~}}}}
    \newcommand{\SSS}[1]{\textcolor{darkorange}{{\bf SS:}{\small \texttt{~#1~}}}}
    \newcommand{\MT}[1]{\textcolor{blue-violet}{{\bf MT:}{\small \texttt{~#1~}}}}
    \newcommand{\FV}[1]{\textcolor{brown}{{\bf FV:}{\small \texttt{~#1~}}}}
\fi

\usepackage{hyperref}

\usepackage[final]{neurips_2024}

\usepackage{amsmath}
\usepackage{amssymb}
\usepackage{mathtools}
\usepackage{amsthm}
\usepackage{amsfonts}
\usepackage{mathrsfs}

\usepackage{stmaryrd}

\usepackage[capitalize,noabbrev]{cleveref}

\theoremstyle{plain}
\newtheorem{theorem}{Theorem}[section]
\newtheorem{proposition}[theorem]{Proposition}
\newtheorem{lemma}[theorem]{Lemma}
\newtheorem{corollary}[theorem]{Corollary}
\theoremstyle{definition}

\theoremstyle{remark}

\usepackage[textsize=tiny]{todonotes}

\nc{\bbR}{\mathbb{R}}
\nc{\bbN}{\mathbb{N}}
\nc{\R}{\bbR}

\nc{\bs}[1]{\boldsymbol{#1}}

\nc{\scs}{[N]}
\nc{\sab}[1]{\square}
\nc{\emp}[1]{\emph{#1}}
\nc{\sss}{\mathcal{S}}
\nc{\ose}{\textsc{CBA}}
\nc{\indi}[1]{\llbracket#1\rrbracket}
\nc{\la}{\leftarrow}
\nc{\sts}{\mathcal{Q}}
\nc{\obe}{\textsc{CBAB}}
\nc{\expfour}{\textsc{Exp4}}
\nc{\rplus}{\mathbb{R}_+}
\nc{\balls}{\mathcal{U}}
\nc{\ball}[2]{\mathcal{B}_{#1,#2}}
\nc{\sbs}[3]{b_{#1,#2,#3}}
\nc{\bal}{\mathcal{B}}
\nc{\sds}[2]{\mathcal{D}_{#1,#2}}
\nc{\anc}[1]{{\Uparrow}(#1)}
\nc{\des}[1]{{\Downarrow}(#1)}
\nc{\sff}[1]{\mathcal{F}_{x,a}(#1)}
\nc{\lc}[1]{\triangleleft(#1)}
\nc{\rc}[1]{\triangleright(#1)}
\nc{\qdr}{\mathcal{Z}}
\nc{\ipr}[2]{\langle#1,#2\rangle}
\nc{\xval}[1]{\mathcal{V}_{#1}}
\nc{\prw}{\tilde{w}}
\nc{\egr}{\hat{\ell}}
\nc{\hr}{\hat{r}} 
\nc{\exg}{{g}}
\nc{\tff}{{f}}
\nc{\expt}[1]{\mathbb{E}[#1]}
\nc{\rtu}{\rho}
\nc{\rtf}{\rtu_{x_t,a}(u)}
\nc{\com}{\mathcal{C}}
\nc{\Ksq}{[K_\square]}
\nc{\scO}{\mathcal{O}}
\nc{\uw}[1]{u_{#1}}
\nc{\sexp}{\textsc{SpecialistExp}}
\nc{\cvs}{\mathcal{C}}

\nc{\ty}{\tilde{y}}
\nc{\mf}{\mathcal{E}}
\nc{\hy}{h}
\nc{\uc}{p}
\nc{\q}{q}

\nc{\sset}{\mathcal{S}}
\nc{\con}{\mathcal{X}}
\nc{\sas}{\mathcal{A}}
\nc{\sa}{\bs{s}}
\nc{\sac}[1]{s_{#1}}
\nc{\pol}{\bs{\pi}}
\nc{\lcp}[1]{\bs{\pi}^{#1}}
\nc{\vlc}{\mathcal{V}(\sset)}
\nc{\pot}[1]{\bs{\pi}_{#1}}
\nc{\seq}[2]{\langle#1\,|\,#2\rangle}
\nc{\ex}[1]{\bs{e}^{#1}}
\nc{\ext}[2]{\bs{e}^{#1}_{#2}}
\nc{\lcm}[1]{\bs{\pi}(#1)}
\nc{\lcc}[2]{\bs{\pi}_{#2}(#1)}
\nc{\val}{\mathcal{V}}
\nc{\bsc}[3]{e^{#1}_{#2,#3}}
\nc{\cnf}[2]{c_{#1,#2}}
\nc{\cnb}[1]{\bs{c}_{#1}}
\nc{\onorm}[1]{\|#1\|_1}
\nc{\snw}[1]{\tilde{\bs{w}}_{#1}}
\nc{\snc}[2]{\tilde{w}_{#1,#2}}
\nc{\lm}{\lambda}
\nc{\sat}[1]{\bs{s}_{#1}}
\nc{\satc}[2]{s_{#1,#2}}
\nc{\hrv}[1]{\hat{\bs{r}}_{#1}}
\nc{\hrc}[2]{\hat{r}_{#1,#2}}
\nc{\epd}[2]{\bs{p}^{#1}_{#2}}
\nc{\draw}{\sim}

\DeclareMathOperator*{\argmin}{arg\,min}
\DeclareMathOperator{\conv}{conv}

\title{Bandits with Abstention under Expert Advice}

\author{
Stephen Pasteris\textsuperscript{1}$^*$\quad Alberto Rumi\textsuperscript{2,3}\quad Maximilian Thiessen\textsuperscript{4}\quad Shota Saito\textsuperscript{5} \\
\textbf{Atsushi Miyauchi\textsuperscript{3}\quad Fabio Vitale\textsuperscript{3} \quad Mark Herbster\textsuperscript{5}}\\
\textsuperscript{1}The Alan Turing Institute\quad\textsuperscript{2}University of Milan\quad 
\textsuperscript{3}CENTAI Institute\\
\textsuperscript{4}TU Wien\quad
\textsuperscript{5}University College London\\
\texttt{$^*$spasteris@turing.ac.uk}
}

\begin{document}\maketitle

\begin{abstract}
We study the classic problem of prediction with expert advice under bandit feedback. 
Our model assumes that one action, corresponding to the learner's abstention from play, has no reward or loss on every trial. 
We propose the \emph{confidence-rated bandits with abstentions} (\ose) 
algorithm, which exploits this assumption to obtain reward bounds that can significantly improve those of the classical \expfour\ algorithm. 
Our problem can be construed as the aggregation of confidence-rated predictors, with the learner having the option to abstain from play.
We are the first to achieve bounds on the expected cumulative reward for general confidence-rated predictors. In the special case of specialists we achieve a novel reward bound, significantly improving the previous bounds of \textsc{SpecialistExp} (treating abstention as another action).
We discuss how \ose\ can be applied to the problem of adversarial contextual bandits with the option of abstaining from selecting any action. We are able to leverage a wide range of inductive biases, outperforming previous approaches both theoretically and in preliminary experimental analysis. Additionally, we achieve a reduction in runtime from quadratic to almost linear in the number of contexts for the specific case of metric space contexts. 


\end{abstract}

\section{Introduction}\label{s:intro}


\SSS{@everyone TODO: add acknowledgment, someone we need to thank (Nicolo?) and funding sources}

\SSS{@stephen TODO: Oh no - I wrote the problem as if the adversary need not be fully oblivious - but it does in fact need to be. I know it’s incredibly minor and is easily fixed but I feel bad - should have known}

\SSS{@stephen? TODO: Incorporating the new prop and proofs}

We study the classic problem of prediction with expert advice under bandit feedback. The problem is structured as a sequence of trials. During each trial, each expert recommends a probability distribution over the set of possible actions. The learner then selects an action and observes and incurs the (potentially negative) reward associated with that action on that particular trial. In practical applications, errors often lead to severe consequences, and consistently making predictions is neither safe nor economically practical. For this reason, the abstention option has gained a lot of interest in the literature, both in the batch and online setting \citep{chow1957optimum,chow1970optimum,hendrickx2021machine, cortes2018online}.
Similarly to previous works, this paper is based on the assumption that one of the actions always has zero reward: such an action is equivalent to an abstention of the learner from play. Besides the rewards being bounded, we make no additional assumptions regarding how the rewards or expert predictions are generated. In this paper, we present an efficient algorithm \ose\ (\underline{\textbf{C}}onfidence-rated \underline{\textbf{B}}andits with \underline{\textbf{A}}bstentions) which exploits the abstention action to get reward bounds that can be dramatically higher than those of \expfour\ \citep{auer2002nonstochastic}. In the worst case, our reward bound essentially matches that of \expfour\ so that \ose\ can be seen as a strict improvement, since the time-complexities of the two algorithms are, up to a factor logarithmic in the time horizon, identical in the general case.

Our problem can also be seen as that of aggregating \emp{confidence-rated predictors}~\citep{blum2007external,gaillard2014second,luo2015achieving} when the learner has the option of abstaining from taking actions.
When the problem is phrased in this way, at the start of each trial, each predictor recommends a probability distribution over the actions (which now may not include an action with zero reward) but with a confidence rating. A low confidence rating can mean that either the predictor thinks that all actions are bad (so that the learner should abstain) or simply does not know which action is the best. Previous works on confidence-rated experts measure the performance of their algorithm in terms of the sum of \emp{scaled} per-trial rewards. In contrast to previous algorithms, our approach allows for the derivation of bounds on the expected cumulative reward of \ose. 

This formulation enables us to extend our work to the problem of adversarial contextual bandits with the abstention option, which has not been studied before. Previous work has considered the abstention option in the standard (context-free) adversarial bandit setting or in stochastic settings \citep{cortes2018online,cortes2020online,neu2020fast}, but not in the contextual and adversarial case. Moreover, their results and methods cannot be applied to confidence-rated predictors.
To get more intuition on this setup, we can think of any deterministic policy that maps contexts into actions. Any such policy can be viewed as a classifier, with \emph{foreground} classes associated with each action and a \emph{background} class associated with abstaining. Our learning bias is represented by a set of information we refer to as the \emph{basis}, which we formally define later. It encodes contextual structural assumptions that hold exclusively for the foreground classes and are provided to the algorithm a priori. A particular type of basis is generated by a set of potential clusters that can overlap. Alternatively, a basis can also be created using balls generated by any kind of distance function, which groups contexts believed to be close together. For this latter family of basis, we can also achieve a significant speedup in the per-trial time complexity of \ose. This result is very different (and incomparable) to other results about adversarial bandits in metric spaces \citep{Pasteris2023NearestNW, Pasteris2023AHN}.

\subsection{Additional related work}

The non-stochastic multi-armed bandit problem, initially introduced by  \citet{auer2002nonstochastic}, has been a subject of significant research interest. \citet{auer2002nonstochastic} also considered the multi-armed bandit problem with expert advice, introducing the \expfour\ algorithm. \expfour\ evolved the field of multi-armed bandits to encompass more complex scenarios, particularly the contextual bandit~\citep{lattimore2020bandit}. Contextual bandits are an extension of the classical multi-armed bandit framework, where an agent makes a sequence of decisions while taking into account contextual information. 
Our work is also related to the multi-class classification with bandit feedback, called \textit{weak reinforcement}~\citep{auer1999structural}.
An action in our bandit setting corresponds to a class in the multi-class classification framework.

As discussed in the introduction, a key aspect of this work is the option to abstain from making any decision. 
In the batch setting~\citep{chow1957optimum,chow1970optimum}, this option is usually referred to as ``rejection''. 
These works study whether to use or reject a specific model prediction based on specific requests (see~\citet{hendrickx2021machine} for a survey). 
In online learning, ``rejection'' can be the possibility of abstention by the learner. 
These works usually rely on a cost associated with the abstention action.
~\citet{neu2020fast} studied the magnitude of the cost associated with abstention in an expert setting with bounded losses. 
They state that if the cost is lower than half of the amplitude of the interval of the loss, it is possible to derive bounds that are independent of the time. 
In \citet{cortes2018online}, a non-contextual and partial information setting with the option of abstention is studied. 
The sequel model~\citep{cortes2020online} regards this model as a special case of their stochastic feedback graph model. 
\citet{schreuder2021classification} studied the fairness setting when using the option of abstaining as it may lead to discriminatory predictions.

One specific scenario where prior algorithms can establish cumulative reward bounds is as follows: 
on any given trial, the predictors are \emph{specialists}~\citep{freund1997using}, having either full confidence (a.k.a. \emph{awake}) or no confidence (a.k.a. \emph{asleep}).
The \textsc{SpecialistExp} algorithm by~\citet{herbster2021gang}, a bandit version of the standard specialist algorithm, achieves regret bounds with respect to any subset of specialists where exactly one specialist is awake on each trial.  
We differ from this work as abstention is an algorithmic choice. 
Instead of sleeping in the rounds where the specialist is not active, the specialist will vote for abstention, which is a proper action of our algorithm. 
In Section~\ref{sec:efficientlearning}, we present an illustrative problem involving learning balls in a space equipped with a metric. 
This example demonstrates our capability to significantly improve on \textsc{SpecialistExp}. 
For this problem, we also present subroutines that significantly speed up \ose.

\section{Problem formulation and notation}
We consider the classic problem of prediction with expert advice under bandit feedback. In this problem we have $K+1$ \emp{actions}, $E$ \emp{experts}, and $T$ \emp{trials}. On each trial $t$:
\begin{enumerate}
\setlength{\itemsep}{0pt}
    \item Each expert suggests, to the learner, a probability distribution over the $K+1$ actions.
    \item The learner selects an action $a_t$.
    \item The reward incurred by action $a_t$ on trial $t$ (which is in $[-1,1]$) is revealed to the learner.
\end{enumerate}
We note that the experts' suggestions and the rewards (associated with each action) are chosen a-priori and hence do not depend on the learner's actions. The aim of the learner is to maximize the cumulative reward obtained by its selected actions. As discussed in Section~\ref{s:intro}, we consider the case in which there is an action (the abstention action) that incurs zero reward on every trial.

We denote our action set by $[K]\cup\{\square\}$ where $\square$ is the abstention action. For each trial $t\in[T]$ we define the vector $\bs{r}_t\in[-1,1]^K$ such that for all $a\in[K]$\,, $r_{t,a}$ is the reward obtained by action $a$ on trial $t$. Moreover, we define $r_{t,\square}:=0$ which is the reward of the abstention action $\square$. 

It will be useful for us to represent probability distributions over the actions by vectors in the set:
\begin{equation*}
\sas:=\{\sa\in[0,1]^K\,|\,\onorm{\sa}\leq1\}\,.
\end{equation*}
Any vector $\sa\in\sas$ represents the probability distribution over actions which assigns, for all $a\in[K]$, a probability of $\sac{a}$ to action $a$, and assigns a probability of $1-\onorm{\sa}$ to the abstention action $\square$, where $\onorm{\sa}$ denotes 1-norm of $\sa$. We write $a\draw \sa$ to represent that action $a$ is drawn from the probability distribution $\sa$. We will refer to the elements of the set $\sas$ as \emp{stochastic actions}.

\SSS{Do we have better notation of this policy, since this is confusing for transpose or $T$-th power? $\sas_{T}$? $\sas^{(T)}$? SP: I think we should leave it as is as anything different would be even more confusing imo.}

A \emp{policy} is any element of $\sas^T$ (noting that any such policy is a matrix in $[0,1]^{T\times K}$). Any policy $\bs{e}\in\sas^T$ defines a stochastic sequence of actions: on every trial $t\in[T]$ an action $a\in[K]\cup\{\square\}$ being drawn as $a\draw \bs{e}_t$. Note that if the learner plays according to a policy $\bs{e}\in\sas^T$, then on each trial $t$ it obtains an expected reward of $\bs{r}_t\cdot\bs{e}_t$, where the operator $\cdot$ denotes the dot product.
Note that each expert is equivalent to a policy. 
Thus, for all $i\in[E]$ we denote the $i$-th expert by $\ex{i}\in\sas^T$.
Hence, at the start of each trial $t\in[T]$, the learner views the sequence $\seq{\ext{i}{t}}{i\in[E]}$.

We can also view the experts as \emp{confidence-rated predictors} over the set $[K]$: for each $i\in[E]$ and $t\in[T]$, the vector $\ext{i}{t}$ can be viewed as suggesting the probability distribution $\ext{i}{t}/\onorm{\ext{i}{t}}$ over $[K]$, but with confidence $\onorm{\ext{i}{t}}$. We denote this confidence by $\cnf{t}{i}:=\onorm{\ext{i}{t}}$ and write $\cnb{t}:=(\cnf{t}{1},\dots,\cnf{t}{E})$.

In this work, we will refer to the \emp{unnormalized relative entropy} defined by:
\begin{equation*}\Delta(\bs{u},\bs{v}):=\sum_{i\in[E]}u_i\ln\left(\frac{u_i}{v_i}\right)-\|\bs{u}\|_1+\|\bs{v}\|_1
\end{equation*}
for any $\bs{u},\bs{v}\in\mathbb{R}_+^E$. We will also use the Iverson bracket notation $\indi{\mathrm{\textsc{Pred}}}$ as the indicator function, meaning that it is equal to $1$ if $\textsc{Pred}$ is true, and $0$ otherwise. All the proofs are in the Appendix.

\nc{\trf}{\rho}

\section{Main result}

Our main result is represented by a bound on the cumulative reward of our algorithm \ose. 
We note that any \emp{weight} vector $\bs{u}\in\mathbb{R}_+^E$ induces a matrix $\lcm{\bs{u}}\in\mathbb{R}_+^{T\times K}$ defined by
\begin{equation*}
    \lcm{\bs{u}}:=\sum_{i\in[E]}u_i\ex{i},
\end{equation*}
which is the linear combination of the experts with coefficients given by $\bs{u}$. 
However, only some of such linear combinations generate valid policies. Thus, we define 
\begin{equation*}
    \val:=\{\bs{u}\in\mathbb{R}_+^E\,|\,\lcm{\bs{u}}\in\sas^T\}
\end{equation*}
as the set of all weight vectors that generate valid policies. Particularly, note that $\bs{u}\in\val$ if and only if, on every trial $t$, the weighted sum of the confidences $\bs{u}\cdot\cnb{t}$ is no greater than one.
Given some $\bs{u}\in\val$, we define 
\begin{equation*}
    \trf(\bs{u}):=\sum_{t\in[T]}\bs{r}_t\cdot\lcc{\bs{u}}{t}\,,
\end{equation*} 
which would be the expected cumulative reward of the learner if it was to follow the policy $\lcm{\bs{u}}$.
We point out that the learner does not know $\val$ or the function $\bs{\pi}$ a-priori.

The following theorem (proved in \Cref{sec:analysis}) allows us to bound the regret of \ose\ with respect to any valid linear combination $\bs{u}$ of experts.

\begin{theorem}\label{cbath}
\ose\ takes parameters $\eta\in(0,1)$ and $\bs{w}_1\in\mathbb{R}_+^E$\,.
    For any $\bs{u}\in\val$ the expected cumulative reward of \ose\ is bounded below by:
    \begin{equation*}
    \sum_{t\in[T]}\mathbb{E}[r_{t,a_t}]\geq\expt{\trf(\bs{u})}- \frac{\Delta(\bs{u},\bs{w}_1)}{\eta}-\eta (12K+2)T\,,
\end{equation*}
where the expectations are with respect to the randomization of \ose's strategy.
The per-trial time complexity of \ose\ is in $\mathcal{O}(KE)$.
\end{theorem}

We now compare our bound to those of previous algorithms. 
\SSS{when we say ``firstly'', we may have the second one. what is that? AR: Addressed (is next paragraph)}
Firstly, \expfour\ can only achieve bounds relative to a $\bs{u}\in\val$ with $\onorm{\bs{u}}=1$\,, in which case it essentially matches our bound but with $12K+2$ replaced by $8K+8$. Hence, for any $\bs{u}\in\val$ the \expfour\ bound essentially replaces the term $\trf(\bs{u})$ in our bound by $\trf(\bs{u})/\onorm{\bs{u}}$. Note that $\onorm{\bs{u}}$ could be as high as the number of experts which implies we can dramatically outperform \expfour\footnote{Precisely, if for each expert there exists a trial in which the confidence is 1, then we have $0 \leq \|u\|_{1} \leq E$. Otherwise can be high as $0 \le \|u\|_1 \leq E/c^*$, where $c^* = \max_{t \in [T]} c_t^i$.}.

Secondly, when viewing our experts as confidence-rated predictors, we note that previous algorithms for this setting only give bounds on a weighted sum of the per-trial rewards where the weight on each trial is $\bs{u}\cdot\cnb{t}$ for some $\bs{u}\in\val$. This is only a cumulative reward bound when $\bs{u}\cdot\cnb{t}=1$ for all $t\in[T]$, and finding such a $\bs{u}$ is typically impossible. When there does exist $\bs{u}$ that satisfies this constraint, the reward relative to $\bs{u}$ is essentially the same as for us \citep{blum2007external}. However, there will often be another value of $\bs{u}$ that will give us a much better bound, as we show in Section \ref{sec:efficientlearning}.


\section{The \ose\ algorithm}\label{sec:cbaalg}

\nc{\pqu}{\mathcal{P}}
\nc{\vat}[1]{\mathcal{V}_{#1}}
\nc{\rf}{\rho_t}
\nc{\gv}[1]{\bs{g}_{#1}}
\nc{\gc}[2]{g_{#1,#2}}

The \ose\ algorithm is given in Algorithm 1. In this section, we describe its derivation via a modification of the classic \emp{mirror descent} algorithm.

\begin{algorithm}[t]
\caption{\ose$(\bs{w}_1, \eta)$}
\label{alg:OSE4}
\textbf{For} $t=1,2,\ldots, T$ \textbf{do}:
\begin{enumerate}
\setlength{\itemsep}{0pt}
\item For all $i\in[E]$ receive $\ext{i}{t}$
\item For all $i\in[E]$ set $\cnf{t}{i}\la\onorm{\ext{i}{t}}$
\item \textbf{If} $\onorm{\cnb{t}}\leq 1$ \textbf{then}:
\begin{enumerate}
\setlength{\parskip}{0pt}
\setlength{\itemsep}{0pt}
    \item Set $\snw{t}\la\bs{w}_t$
\end{enumerate}
\item \textbf{Else}:
\begin{enumerate}
\setlength{\itemsep}{0pt}
\item By interval bisection find $\lm>0$ such that:
\begin{equation*}
    \sum_{i\in[E]}\cnf{t}{i}{w}_{t,i}\exp(-\lm\cnf{t}{i})=1
\end{equation*}
\item For all $i\in[E]$ set $\snc{t}{i}\la w_{t,i}\exp(-\lm \cnf{t}{i})$
\end{enumerate}
\item Set: 
\begin{equation*}
    \sat{t}\la\sum_{i\in[E]}\snc{t}{i}\ext{i}{t}
\end{equation*}
\item Draw $a_t\draw\sat{t}$
\item Receive $r_{t,a_t}$
\item For all $a\in[K]$ set: 
\begin{equation*}
    \hrc{t}{a}\la 1-\indi{a=a_t}(1-r_{t,a_t})/\satc{t}{a_t}
\end{equation*}
\item For all $i\in[E]$ set $w_{(t+1),i}\la\snc{t}{i}\exp(\eta \ext{i}{t}\cdot\hrv{t})$
\end{enumerate}
\end{algorithm}

Our modification of mirror descent is based on the following mathematical objects. For all $t\in[T]$ we first define:
\begin{equation*}
\vat{t}:=\{\bs{v}\in\mathbb{R}_+^E\,|\,\onorm{\lcc{\bs{v}}{t}}\leq 1\}\,,
\end{equation*}
which is the set of all weight vectors that give rise to linear combinations producing valid stochastic actions at trial $t$. 
Given some $t\in[T]$, we define our \emp{objective function} $\rf:\vat{t}\rightarrow[-1,1]$ as
\begin{equation*}
     \rf(\bs{v}):=\bs{r}_t\cdot\lcm{\bs{v}} \text{ for all } \bs{v}\in\vat{t}.
\end{equation*}

Like mirror descent, \ose\ maintains, on each trial $t\in[T]$, a weight vector $\bs{w}_t\in\mathbb{R}_+^E$. However, unlike mirror descent on the simplex, we do not keep $\bs{w}_t$ normalized, but we will instead project it into $\vat{t}$ at the start of trial $t$, producing a vector $\snw{t}$. Also, unlike mirror descent, \ose\ does not use the actual gradient (which it does not know) of $\rf$ at $\snw{t}$, but (inspired by the \textsc{Exp3} algorithm) uses an unbiased estimator instead. Specifically, on each trial $t\in[T]$\,, \ose\ does the following:
\begin{enumerate}
\setlength{\itemsep}{0pt}
    \item Set $\snw{t}\la\operatorname{argmin}_{\bs{v}\in\vat{t}}\Delta(\bs{v},\bs{w}_t)$. 
    \item Randomly construct a vector $\gv{t}\in\mathbb{R}^E$ such that $\mathbb{E}[\gv{t}]=\nabla\rf(\snw{t})$. 
    \item Set $ \bs{w}_{t+1}\la\operatorname{argmin}_{\bs{v}\in\mathbb{R}_+^E}(\eta\gv{t}\cdot(\snw{t}-\bs{v})+\Delta(\bs{v},\snw{t}))$. 
\end{enumerate}
This naturally raises two questions: how is $a_t$ selected and how is $\gv{t}$ constructed? On each trial $t\in[T]$ we define
\begin{equation*}
\sat{t}:=\sum_{i\in[E]}\snc{t}{i}\ext{i}{t}\,,
\end{equation*}
which is the stochastic action generated by the linear combination $\snw{t}$, and select $a_t\draw\sat{t}$. Note that:
\begin{equation}\label{algseceq1}
\mathbb{E}[r_{t,a_t}]=\rf(\snw{t})\,,
\end{equation}
which confirms that $\rf$ is our objective function at trial $t$. Once $r_{t,a_t}$ is revealed to us we can proceed to construct the gradient estimator $\gv{t}$. It is important that we construct this estimator in a specific way. Inspired by \expfour\ we first define a reward estimator $\hrv{t}$ such that for all $a\in[K]$ we have:
\begin{equation*}
    \hrc{t}{a}:= 1-\indi{a=a_t}(1-r_{t,a_t})/\satc{t}{a_t}\,.
\end{equation*}
This reward estimate is unbiased as:
\begin{equation*}
\mathbb{E}[\hrc{t}{a}]=1-\Pr[a=a_t](1-r_{t,a})/\satc{t}{a}=r_{t,a}\,.
\end{equation*}
We then define, for all $i\in[E]$, the component:
\begin{equation*}
\gc{t}{i}:=\ext{i}{t}\cdot\hrv{t}\,.
\end{equation*}
Note that for all $i\in[E]$ we have:
\begin{equation*}
\mathbb{E}[\gc{t}{i}]=\ext{i}{t}\cdot\mathbb{E}[\hrv{t}]=\ext{i}{t}\cdot\bs{r}_t=\partial_i\rf(\snw{t})
\end{equation*}
so that $\mathbb{E}[\gv{t}]=\nabla\rf(\snw{t})$ as required.

Now that we defined the process by which \ose\ operates we must show how to compute $\snw{t}$ and $\bs{w}_{t+1}$. First we show how to compute $\snw{t}$ from $\bs{w}_t$. If $\onorm{\cnb{t}}\leq 1$ it holds that $\bs{w}_t\in\vat{t}$ so we immediately have $\snw{t}=\bs{w}_t$. Otherwise we must find $\snw{t}\in\mathbb{R}_+^E$ that minimizes $\Delta(\snw{t},\bs{w}_t)$ subject to the constraint: 
\begin{equation*}
    \sum_{i\in[E]}\snc{t}{i}\cnf{t}{i}=1\,,
\end{equation*}
which is equivalent to the constraint that $\onorm{\lcm{\snw{t}}}=1$.
Hence, by Lagrange's theorem there exists $\lm$ such that:
\begin{equation*}
\nabla_{\snw{t}}\biggl(\Delta(\snw{t},\bs{w}_t)+\lm\sum_{i\in[E]}\snc{t}{i}\cnf{t}{i}\bigg)=0
\end{equation*}
which is solved by setting, for all $i\in[E]$\,: 
\begin{equation*}
    \snc{t}{i}:= w_{t,i}\exp(-\lm \cnf{t}{i})\,.
\end{equation*}
The constraint is then satisfied if $\lm$ is such that:
\begin{equation*}
    \sum_{i\in[E]}\cnf{t}{i}{w}_{t,i}\exp(-\lm\cnf{t}{i})=1\,.
\end{equation*}
Since this function is monotonic decreasing, $\lm$ can be found by interval bisection. For this computation step, we treat our numerical precision as a constant in our time complexity. In Appendix \ref{apx:precision}, we show that, even if the numerical precision is unbounded, we incur a time complexity equal to that of \textsc{Exp4}, up to a factor logarithmic in $T$, adding only 1 to the regret. \AR{Changed this and pointed to appendix}

Turning to the computation of $\bs{w}_{t+1}$\,, since it is unconstrained it is found by the equation:
\begin{equation*}
\nabla_{\bs{w}_{t+1}}(\gv{t}\cdot\bs{w}_{t+1}+\eta^{-1}\Delta(\bs{w}_{t+1},\snw{t}))=0\,.
\end{equation*}
which is solved by setting, for all $i\in[E]$\,: 
\begin{equation}\label{algseceq2}
w_{(t+1),i}:=\snc{t}{i}\exp(\eta\gc{t}{i})\,.
\end{equation}

\nc{\bll}{B}
\nc{\bra}{\delta}
\nc{\cra}[1]{b_{#1}}
\nc{\cons}{\mathcal{X}}
\nc{\basis}{\mathcal{B}}

\section{Adversarial contextual bandits with abstention}

One main application of \ose\ is in the problem of adversarial contextual bandits with a finite context set. In this problem, we have a finite set of \emp{contexts} $\cons$. A-priori nature selects a sequence:
\begin{equation*}
\seq{(x_t,\bs{r}_t)\in\cons\times[-1,1]^K}{t\in[T]}\,,
\end{equation*}
but does not reveal it to the learner. For all $t\in[T]$ we define $r_{t,\square}:=0$. On each trial $t\in[T]$ the following happens:
\begin{enumerate}
\setlength{\itemsep}{0pt}
\item The context $x_t$ is revealed to the learner.
\item The learner selects an action $a_t\in[K]\cup\{\square\}$.
\item The learner sees and incurs reward $r_{t,a_t}\in[-1,1]$.
\end{enumerate}

We will assume that we are given, a-priori, a set $\basis\subseteq2^\cons$ that we call the \emph{basis}. We call each element of $\basis$ a \emph{basis element} (which is a set of contexts). We will later introduce various potential bases, determined by the nature of the context's structure: points within a metric space, nodes within a graph, and beyond. Importantly, our method is capable of accommodating any type of basis and, thus, any potential inductive bias that might be present in the data.

Given our basis we run our algorithm \ose\ with each expert corresponding to a pair $(\bll,k)\in\basis\times[K]$. The expert corresponding to each pair $(\bll,k)$ will deterministically choose action $k$ when the current context $x_t$ is in $\bll$, and abstain otherwise.

\begin{corollary}\label{basisthm}
    Given any basis $\basis$ of cardinality $N$ and any $M\in\mathbb{N}$ we can implement \ose\ such that for any sequence of disjoint basis elements $\seq{\bll_j}{j\in[M]}$ with corresponding actions $\seq{\cra{j}\in[K]}{j\in[M]}$ we have:
    \begin{align*}
\sum_{t\in[T]}\expt{r_{t,a_t}}
\geq\sum_{t\in[T]}\sum_{j\in[M]}\indi{x_t\in\bll_j}r_{t,\cra{j}}-\sqrt{2M\ln(N)(6K+1)T}\,.
    \end{align*}
    The per-trial time complexity of this implementation of \ose\ is in $\mathcal{O}(KN)$.
\end{corollary}
\begin{proof}
The choice of experts for \ose\ that leads to Corollary \ref{basisthm} is defined by the set of pairs so that $E=NK$ and for each $\bll\in\basis$ and action $a\in[K]$ there exists an unique $i\in[E]$ such that for all $t\in[T]$ and $b\in[K]$ we have:
\begin{equation*}
    e^i_{t,b}:=\indi{x_t\in\bll}\indi{b=a}\,.
\end{equation*}
By choosing $w_{1,i}:=M/NK$ for all $i\in[E]$\,, and choosing
\begin{equation*}
    \eta:=(M\ln(N)/(6K+1)T)^{-1/2}\,,
\end{equation*}
Theorem \ref{cbath} implies the reward bound in Corollary \ref{basisthm}.
The per-trial time complexity of a direct implementation of \ose\ for this set of experts would be $\mathcal{O}(KN)$.
\end{proof}

We briefly comment on the term:
\begin{equation*}
    \sum_{j\in[M]}\indi{x_t\in\bll_j}r_{t,\cra{j}}\,,
\end{equation*}
that appears in the theorem statement. If $x_t$ does not belong to any of the sets in $\seq{\bll_j}{j\in[M]}$ then this term is equal to zero (which is the reward of abstaining). Otherwise, since the sets are disjoint, $x_t$ belongs to exactly one of them and the term is equal to the reward induced by the action that corresponds to that set. In other words, the total cumulative reward is bounded relative to that of the policy that abstains whenever $x_t$ is outside the union of the sets and otherwise selects the action corresponding to the set that $x_t$ lies in. 

\SSS{We may clarify foreground/background issues in Fig~\ref{fig:example-abs} as pointed out by one reviewer?}

Note the vast improvement of our reward bound over that of \textsc{SpecialistExp} with abstention as one of the actions. Let's assume our context set is a metric space and our basis is the set of all balls. In order to get a reward bound for \textsc{SpecialistExp}, the sets in which the specialists are awake must partition the set $\cons$. This means that we must add to our $M$ balls a disjoint covering (by balls) of the complement of the union of the original $M$ balls. Note that the added balls correspond to the sets in which the specialists predicting the abstention action are awake. Typically this would require a huge number of balls so that the total number of specialists is huge (much larger than $M$); this huge number of specialists essentially replaces the term $M$ in our reward bound (we illustrate an example in \Cref{fig:example-abs}).

Furthermore, in Appendix \ref{apx:overlapping}, we show that the same implementation of \ose\ is capable of learning a weighted set of \textit{overlapping} basis elements, as long as the sum of the weights of the basis elements covering any context is bounded above by one, which \textsc{SpecialistExp} cannot do in general.

As we will see below, the practical bases we propose have a moderate size of $|\basis|=\scO(|\cons|^2)$ leading to a per-step runtime of $\scO(K|\cons|^2)$ for \ose\ in this contextual bandit problem. In \Cref{sec:efficientlearning}, we show how to significantly improve the runtime for a broad family of bases.
\definecolor{CBBlue}{HTML}{377eb8}
\definecolor{CBOrange}{HTML}{ff7f00}

\begin{figure}
  \centering
     \subfigure[Two foreground classes and background as abstained.]{
        \centering
        \begin{tikzpicture}[every node/.style={circle, draw, fill = white, minimum size=1mm}, scale=0.55]
        \foreach \x in {0,1,2} {
            \foreach \y in {0,1,2} {
                \draw (\x,\y) -- (\x+1,\y);
                \draw (\x,\y) -- (\x,\y+1);
            }
        }
        \foreach \x in {3,4,5} {
            \foreach \y in {3,4,5} {
                \draw (\x,\y) -- (\x+1,\y);
                \draw (\x,\y) -- (\x,\y+1);
            }
        }
        \draw (0,3) -- (1,3);
        \draw (1,3) -- (2,3);
        \draw (2,3) -- (3,3);

        \draw (3,3) -- (3,2);
        \draw (3,1) -- (3,2);
        \draw (3,1) -- (3,0);
        
        \draw (6,3) -- (6,4);
        \draw (6,4) -- (6,5);
        \draw (6,5) -- (6,6);
        
        \draw (5,6) -- (6,6);
        \draw (5,6) -- (4,6);
        \draw (3,6) -- (4,6);
        
        \foreach \x in {0,1,2,3} {
            \foreach \y in {0,1,2,3} {
                \node (\x\y) at (\x,\y) {};
            }
        }
        \foreach \x in {3,4,5,6} {
            \foreach \y in {3,4,5,6} {
                \node (\x\y) at (\x,\y) {};
            }
        }    

        \draw[CBBlue, fill=CBBlue,fill opacity=.2, thick, rounded corners=5mm] (-0.5, -0.5) rectangle (2.5,2.5);
        \draw[CBOrange, fill=CBOrange, fill opacity=.2, thick, rounded corners=5mm] (3.5, 3.5) rectangle (6.5,6.5);
    \end{tikzpicture}
        \label{fig:example1}
    }
    \hspace*{5em}
     \subfigure[Two foreground classes and the background as another one.]{
        \centering
        \begin{tikzpicture}[every node/.style={circle, draw, fill = white, minimum size=1mm}, scale=0.55]
        \foreach \x in {0,1,2} {
            \foreach \y in {0,1,2} {
                \draw (\x,\y) -- (\x+1,\y);
                \draw (\x,\y) -- (\x,\y+1);
            }
        }
        \foreach \x in {3,4,5} {
            \foreach \y in {3,4,5} {
                \draw (\x,\y) -- (\x+1,\y);
                \draw (\x,\y) -- (\x,\y+1);
            }
        }
        \draw (0,3) -- (1,3);
        \draw (1,3) -- (2,3);
        \draw (2,3) -- (3,3);

        \draw (3,3) -- (3,2);
        \draw (3,1) -- (3,2);
        \draw (3,1) -- (3,0);
        
        \draw (6,3) -- (6,4);
        \draw (6,4) -- (6,5);
        \draw (6,5) -- (6,6);
        
        \draw (5,6) -- (6,6);
        \draw (5,6) -- (4,6);
        \draw (3,6) -- (4,6);
        
        \foreach \x in {0,1,2,3} {
            \foreach \y in {0,1,2,3} {
                \node (\x\y) at (\x,\y) {};
            }
        }
        \foreach \x in {3,4,5,6} {
            \foreach \y in {3,4,5,6} {
                \node (\x\y) at (\x,\y) {};
            }
        }    

        \draw[CBBlue!80, fill=CBBlue, fill opacity=.2, thick, rounded corners=5mm] (-0.5, -0.5) rectangle (2.5,2.5);
        \draw[CBOrange!80, fill=CBOrange, fill opacity=.2, thick, rounded corners=5mm] (3.5, 3.5) rectangle (6.5,6.5);

        \draw[black!30, thick, fill=black!30, fill opacity=.2, rounded corners=2mm] (-0.5,3.4) rectangle (0.4,2.6);
        \draw[black!30, thick, fill=black!30, fill opacity=.2, rounded corners=2mm] (1.4,3.4) rectangle (0.6,2.6);
        \draw[black!30, thick, fill=black!30, fill opacity=.2,  rounded corners=2mm] (1.5,3.4) rectangle (2.4,2.6);
        \draw[black!30, thick, fill=black!30, fill opacity=.2,  rounded corners=2mm] (3.4,3.4) rectangle (2.6,2.6);
        \draw[black!30, thick, fill=black!30, fill opacity=.2,  rounded corners=2mm] (3.5,3.4) rectangle (4.4,2.6);
        \draw[black!30, thick, fill=black!30, fill opacity=.2,  rounded corners=2mm] (4.5,3.4) rectangle (5.4,2.6);
        \draw[black!30, thick, fill=black!30, fill opacity=.2,  rounded corners=2mm] (5.5,3.4) rectangle (6.5,2.6);

        \draw[black!30, thick, fill=black!30, fill opacity=.2,  rounded corners=2mm] (2.6,6.4) rectangle (3.45,5.55);
        \draw[black!30, thick, fill=black!30, fill opacity=.2,  rounded corners=2mm] (2.6,4.6) rectangle (3.45,5.45);
        \draw[black!30, thick, fill=black!30, fill opacity=.2,  rounded corners=2mm] (2.6,4.45) rectangle (3.45,3.6);
        
        \draw[black!30, thick, fill=black!30, fill opacity=.2,  rounded corners=2mm] (2.6,1.6) rectangle (3.45,2.45);
        \draw[black!30, thick, fill=black!30, fill opacity=.2,  rounded corners=2mm] (2.6,0.6) rectangle (3.45,1.45);
        \draw[black!30, thick, fill=black!30, fill opacity=.2,  rounded corners=2mm] (2.6,-0.5) rectangle (3.45,0.4);

    \end{tikzpicture}
        \label{fig:example2}
    }
    \caption{Illustrative example of abstention where we cover the foreground and background classes with metric balls.
    We consider two clusters (blue and orange) as the foreground and one background class (white), using the shortest path $d_\infty$ metric.
    Using abstention, we can cover two clusters with one ball for each and abstain the background with no balls required (Fig.~\ref{fig:example1}).
    In contrast, if we treat the background class as another class, it would require significantly more balls to cover the background class, as seen by the 10 gray balls in Fig.~\ref{fig:example2}.
    If the number of balls to cover significantly increases like in this case, the bound involving the number of balls also gets significantly worse. 
    }
    \label{fig:example-abs}
\end{figure}
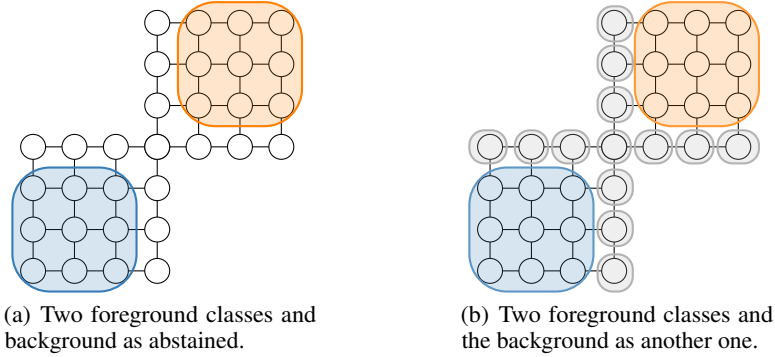
\subsection{A lower bound}


In this section, we show that \ose\ is, up to an $\mathcal{O}(\ln(|\basis|))$ factor, essentially best possible on this contextual bandit problem:

\begin{proposition}
Take any learning algorithm. Given any basis $\basis$ and any $M\in\mathbb{N}$, for any sequence of disjoint basis elements $\seq{\bll_j}{j\in[M]}$ there exists a sequence of corresponding actions $\seq{\cra{j}\in[K]}{j\in[M]}$ such that an adversary can force:
\begin{equation*}
    \sum_{t\in [T]}\sum_{j\in[M]} \indi{x_t \in \mathcal{B}_j}r_{t,b_j} -
    \sum_{t\in [T]} \mathbb{E}[r_{t,a_t}]\in \Omega\left(\sqrt{MKT}\right)\,.
\end{equation*}    
\end{proposition}

\subsection{Efficient learning with balls}\label{sec:efficientlearning}
In practice we can often quantify the similarity between any pair of contexts.
That is, the contexts form a metric space, equipped with a \emp{distance} function $d:\cons\times\cons\rightarrow\mathbb{R}_+$ known to the learner a-priori. For example, contexts could have feature vectors in $\R^m$ (and the metric is the standard Euclidean distance or cosine similarity) or be nodes in a graph with the metric given by the shortest-path distance. A natural basis for this situation is the set of metric \emph{balls}. Specifically, a ball is any set $\bll\subseteq\cons$ in which there exists some $x\in\cons$ and $\bra\in\mathbb{R}_+$ with:
\begin{equation*}
    \bll=\{z\in\cons\,|\,d(x,z)\leq\bra\}\,.
\end{equation*}
For this broad family of bases\footnote{Actually we require a weaker condition. We only use the fact that for each context $z\in\cons$ we have a set $\basis_z = \{B^z_1,\dots, B^z_\ell\}$ of monotonically increasing basis elements, that is, $B^z_i\subseteq B^z_j$ for $i<j$, and the whole basis is formed by the union of these: $\basis = \bigcup_{z\in\cons} \basis_z$.} we can achieve the following speed-up, relying on a a sophisticated data structure based on binary trees.
\begin{theorem}\label{ballth}
    Let $N:=|\cons|$. Given any $M\in\mathbb{N}$ we can implement \ose\ such that for any sequence of disjoint balls $\seq{\bll_j}{j\in[M]}$ with corresponding actions $\seq{\cra{j}\in[K]}{j\in[M]}$ we have:
    \begin{align*}
\sum_{t\in[T]}\expt{r_{t,a_t}}
\geq\sum_{t\in[T]}\sum_{j\in[M]}\indi{x_t\in\bll_j}r_{t,\cra{j}}-\sqrt{4M\ln(N)(6K+1)T}\,.
    \end{align*}
    The per-trial time complexity of this implementation of \ose\ is in $\mathcal{O}(KN\ln(N))$.
\end{theorem}
As there are at most $\scO(N^2)$ metric balls, this improves the runtime of the direct \ose\ implementation from $\scO(KN^2)$ to $\scO(KN\ln(N))$, that is almost linear per step. All the details are in \Cref{apx:efficient_imp}.

\begin{figure*}[!t]
    \centering
    \subfigure[Stochastic Block Model]{%
\includegraphics[width=.245\hsize]{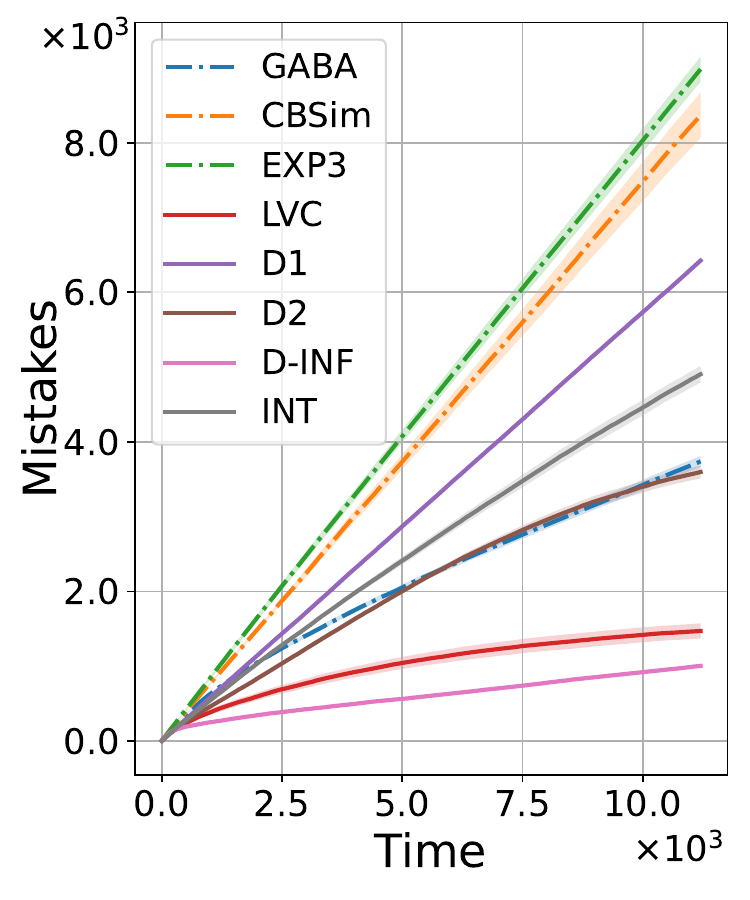}\label{fig:cliquemain}}
\subfigure[Gaussian graph]{%
\includegraphics[width=.245\hsize]{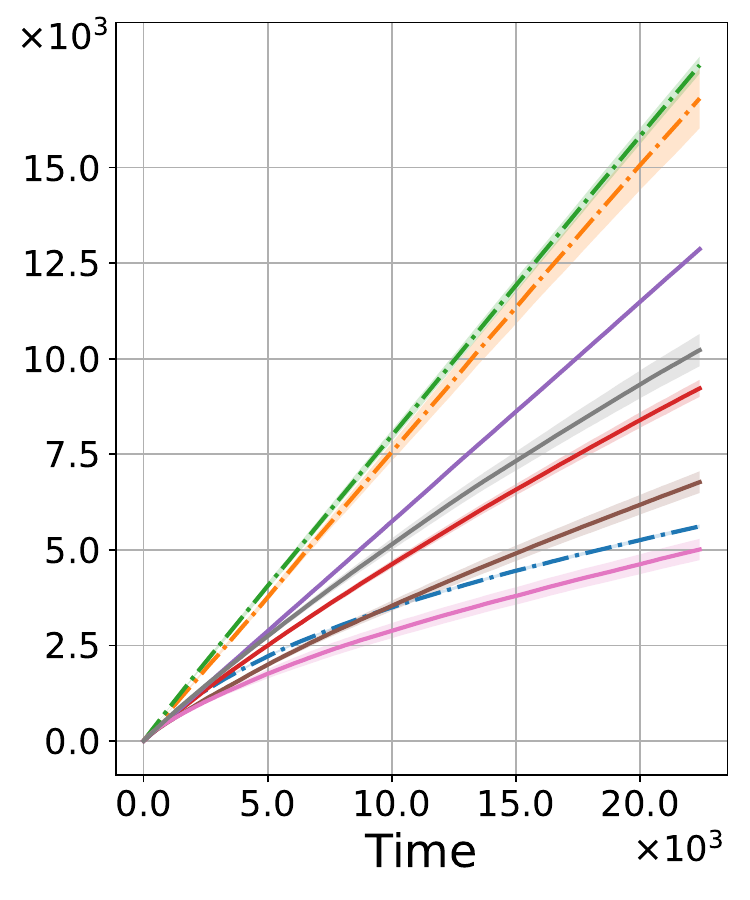}\label{fig:gaussianmain}}
\subfigure[Cora graph]{%
\includegraphics[width=.245\hsize]{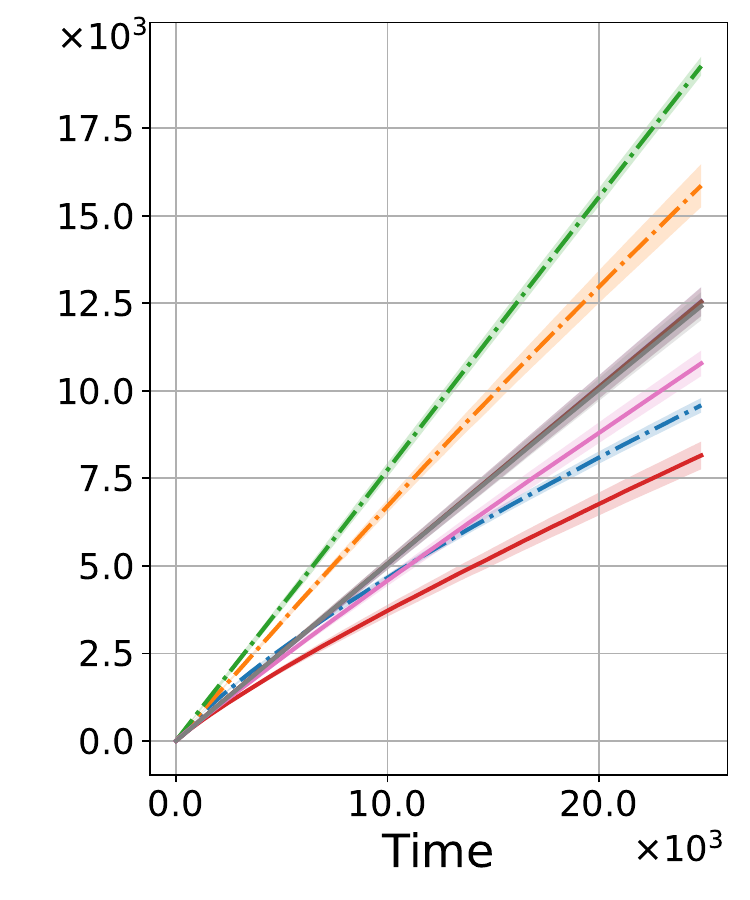}\label{fig:coramain}}
\subfigure[LastFM Asia graph]{%
\includegraphics[width=.245\hsize]{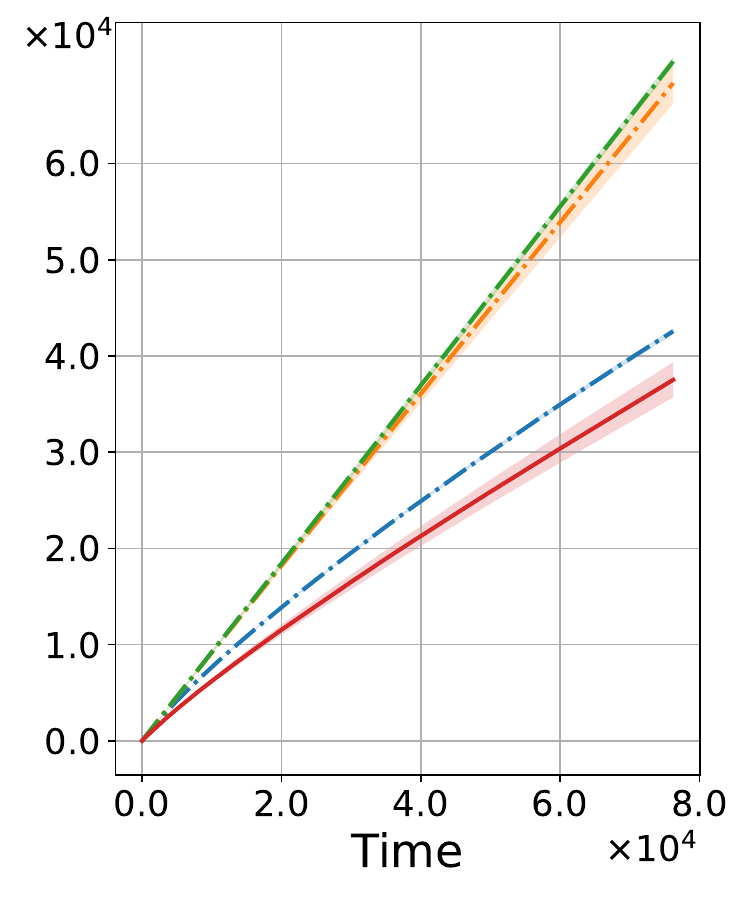}\label{fig:lastfmmain}}
    \caption{Results regarding the number of mistakes over time, the four main settings are presented from left to right: the Stochastic Block Model, Gaussian graph, Cora graph and LastFM Asia graph. In this context, D1, D2, and D-INF represent the $p$-norm bases, LVC represents the community detection basis, and INT represents the interval basis. The baselines, EXP3 for each context, Contextual Bandit with similarity, and GABA-II, are denoted as EXP3, CBSim, and GABA, respectively, and are represented with dashed lines. All the figures display the data with 95\% confidence intervals over 20 runs, calculated using the standard error multiplied by the $z$-score 1.96.}
\end{figure*}

\section{Experiments} 
\label{sec:experiments}

\SSS{Added execuse to use graph data.}



This section conducts preliminary experiments, the code is available at GitHub\footnote{\href{https://github.com/albertorumi/ContextualBanditsWithAbstention}{https://github.com/albertorumi/ContextualBanditsWithAbstention}}. 
We evaluate our method to compare existing algorithms using graph data, since it is common to consider graph structures under the confidence-rated expert setting~\citep{cesa2013gang,herbster2021gang}. 
As mentioned above, the bases used in our algorithm can be constructed arbitrarily, allowing to encompass different inductive biases based on applications. 
Thus, we consider some representative bases used on learning tasks on graphs before, each leading to different inductive priors on the contexts. 
We provide a short description of the bases here and refer to \Cref{sec:thedeialsofthebases} for more details.

\textbf{Effective $p$-resistance basis $d_p$}: Balls given by the metric $$
d_p(i,j) := \biggl({\min\limits_{\substack{\bs{u}\in\R^N\\u_{i}-u_{j}=1}} \sum\limits_{s,t\in V}|u_s - u_t|^p}\biggr)^{-1/p}\,.$$ We use $d_1$, $d_2$, and $d_\infty$ \citep{herbster2009predicting}.

\textbf{Louvain method basis (LVC)}: Communities returned by the Louvain method~\citep{blondel2008fast}, processed by the greedy peeling algorithm~\citep{lanciano2023survey}.

\textbf{Geodesic intervals basis (INT)}: All sets of the form $I(x,y):=\{z\in\cons \mid z \text{ is on a shortest } x\text{-}y\text{ path}\}$ for all $x,y\in\cons$ \citep{pelayo2013geodesic, thiessen2021active}.

Let $N$ be the cardinality of $|\cons|$. For all three basis types, we immediately get an $\scO(KN^2)$ runtime per step of \ose\ as there are $\scO(N^2)$ basis elements. Moreover, for $d_p$ balls and the LVC basis we can use the more efficient $\scO(KN\ln N)$ implementation through Theorem~\ref{ballth}. 
We empirically evaluate our approach in the context of online multi-class node classification on a given graph with bandit feedback. At each time step, the algorithm is presented with a node chosen uniformly at random and must either predict an action from the set of possible actions $[K]$ or abstain. The node can accept (resulting in a positive reward) or reject (resulting in a negative reward) the suggestion based on its preferred class with a certain probability. 
In a real-world application, this models a scenario where each user has a category preference (such as music genre or interest). When the item we decide to present matches their interest, there is a high probability of receiving a reward.

We compare our approach \ose\ using each of these bases on real-world and artificial graphs against the following baselines: an implementation of \textsc{ContextualBandit} from \citet{slivkins2011contextual}, the \textsc{GABA-II} algorithm proposed by \citet{herbster2021gang}, 
and an \textsc{EXP3} instance for each data point.
We use the following graphs for evaluation.


\textbf{Stochastic block model.} We use an established synthetic graph, \textit{stochastic block model}~\citep{holland1983stochastic}. 
This graph is generated by spawning an arbitrary number of disjoint cliques representing the foreground classes. Then an arbitrary number of background points are generated and connected to every possible point with a low probability. Figure~\ref{fig:cliquemain} are displayed the results for the case of $F = 160$ nodes for each foreground class and $B = 480$ nodes for the background class. Connecting each node of the background class with a probability of $1/\sqrt{FB}$.

\textbf{Gaussian graph.} The points on this graph are generated in a two-dimensional space using five different Gaussian distributions with zero mean. Four of them are positioned at the corners of the unit square, representing the foreground classes and having a relatively low standard deviation. Meanwhile, the fifth distribution, representing the background class, is centered within the square and is characterized by a larger standard deviation. The points are linked in a $k$-nearest neighbors graph. In Figure \ref{fig:gaussianmain} are displayed the results for 160 nodes for each foreground class and a standard deviation of 0.2, 480 nodes for the background class with a standard deviation of 1.75, along with a 7-nearest neighbors graph.


\textbf{Real-world dataset.} We tested our approach on the Cora dataset~\citep{sen2008collective} 
and the LastFM Asia dataset~\citep{snapnets}. 
While both of these graphs contain both features and a graph, we exclusively utilized the largest connected component of each graph, resulting in 2485 nodes and 5069 edges for the Cora graph and 7624 nodes and 27806 edges for the LastFM Asia graph.
Subsequently, we randomly chose a subset of three out of the original seven and eighteen classes, respectively, to serve as the background class. Additionally, we selected 15\% of the nodes from the foreground classes randomly to represent noise points, and we averaged the results over multiple runs, varying the labels chosen for noise. Both in Figures \ref{fig:coramain} and \ref{fig:lastfmmain} we averaged over 5 different label sets as noise. For the LastFM Asia graph, we exclusively tested the LVC bases, as it is the most efficient one to compute given the large size of the graph. 

\textbf{Results.}
The results from both synthetically generated tests (Figures \ref{fig:cliquemain} and~\ref{fig:gaussianmain}) demonstrate the superiority of our method when compared to the baselines. 
In particular, $d_\infty$-balls delivered exceptional results for both graphs, 
implying that $d_\infty$-balls effectively cover the foreground classes as expected. 
For the Cora dataset (Figure~\ref{fig:coramain}), we observed that our method outperforms \textsc{GABA-II} only when employing the community detection basis. This similarity in performance is likely attributed to the dataset's inherent lack of noise. Worth noting that the method we employed to inject noise into the dataset may not have been the optimal choice for this specific context. However, it is essential to highlight that our primary focus revolves around the abstention criteria, which plays a central role in ensuring the robustness of our model in the presence of noise. For the LastFM Asia dataset, our objective was to assess the practical feasibility of the model on a larger graph. We tested the LVC bases as they were the most promising and most efficient to compute. We outperform the baselines in our evaluation as shown in Figure \ref{fig:lastfmmain} and further discussed in Appendix~\ref{apx:experiments}.

In summary, our first results confirm what we expected: our approach excels when we choose basis functions that closely match the context's structure. However, it also encounters difficulties when the chosen basis functions are not a good fit for the context.
In Appendix~\ref{apx:experiments}, the results for a wide range of different parameters used to generate the previously described graphs are displayed.


\section*{Acknowledgement}


SP acknowledges the following funding. Research funded by the Defence Science and Technology Laboratory (Dstl) which is an executive agency of the UK Ministry of Defence providing world class expertise and delivering cutting-edge science and technology for the benefit of the nation and allies. The research supports the Autonomous Resilient Cyber Defence (ARCD) project within the Dstl Cyber Defence Enhancement programme.
AR acknowledges the support from the NeurIPS 2024 Financial Assistance. 
MT acknowledges support from a DOC fellowship of the Austrian academy of sciences (ÖAW).
SS acknowledges the support by Huawei for his Ph.D study at UCL.


\bibliography{references}
\bibliographystyle{plainnat}

\newpage
\appendix
\onecolumn

\section{\ose\ analysis}\label{sec:analysis}
Here we prove Theorem~\ref{cbath} from the modification of mirror descent (and the specific construction of $\gv{t}$) given in Section \ref{sec:cbaalg}. Whenever we take expectations in this analysis they are over the draw of $a_t$ from $\sat{t}$ for some $t\in[T]$. As for mirror descent, our analysis hinges on the following classic lemma:

\nc{\cs}{\mathcal{C}}
\nc{\bv}{\bs{v}}
\nc{\bz}{\bs{z}}
\nc{\bu}{\bs{u}}
\nc{\pcf}{\xi}
\nc{\rpe}{\mathbb{R}_+^E}
\nc{\bq}{\bs{q}}
\nc{\cst}{\beta}

\begin{lemma}\label{mainlem}
   Given any convex set $\cs\subseteq\rpe$\,, any convex function $\pcf:\rpe\rightarrow\mathbb{R}$\,, any $\bq\in\cs$ and any $\bz\in\rpe$ with:
   \begin{equation*}       \bq=\operatorname{argmin}_{\bv\in\cs}(\pcf(\bv)+\Delta(\bv,\bz))\,,
   \end{equation*}
   then for all $\bu\in\cs$ we have:
   \begin{equation*}  \pcf(\bu)+\Delta(\bu,\bz)\geq\pcf(\bq)+\Delta(\bu,\bq)\,.
   \end{equation*}
\end{lemma}
\begin{proof}
Theorem~9.12 in \citet{Beck2017first}\, shows that the theorem holds if $\Delta$ is Bregman divergence. In our case $\Delta$ is indeed a Bregman divergence: that of the convex function $f:\rpe\rightarrow\mathbb{R}$ for all $\bv\in\rpe$ defined by:
\begin{equation*}
f(\bv):=\sum_{i\in[E]}v_i\ln(v_i),
\end{equation*}
which concludes the proof.
\end{proof}

\begin{proof}[Proof of Theorem~\ref{cbath}]
Choose any $\bu\in\val$ and $t\in[T]$. We immediately have $\val\subseteq\vat{t}$ by definition, and therefore $\bu\in\vat{t}$. Hence, by setting $\pcf$ such that $\pcf(\bv):=0$ for all $\bv\in\rpe$\,, setting $\cs\in\vat{t}$ and setting $\bz=\bs{w}_t$ in Lemma \ref{mainlem} we have $\bq=\snw{t}$ so that:
\begin{equation}\label{eopeq2}
\Delta(\bu,\bs{w}_t)\geq\Delta(\bu,\snw{t})\,.
\end{equation}
Alternatively, by setting $\pcf$ such that $\pcf(\bv):=\eta\gv{t}\cdot(\snw{t}-\bv)$ for all $\bv\in\rpe$\,, setting $\cs=\rpe$ and setting $\bz=\snw{t}$ in Lemma \ref{mainlem} we have $\bq=\bs{w}_{t+1}$ so that:
\begin{align}
\eta\gv{t}\cdot(\snw{t}-\bu)+\Delta(\bu,\snw{t})
\label{preq1}&\geq\eta\gv{t}\cdot(\snw{t}-\bs{w}_{t+1})+\Delta(\bs{u},\bs{w}_{t+1})\,.
\end{align}
Since $\mathbb{E}[\gv{t}]=\nabla\rf(\snw{t})$ and $\rf$ is linear we have:
\begin{equation}\label{preq2}
    \mathbb{E}[\gv{t}\cdot(\snw{t}-\bu)]=\rf(\snw{t})-\rf(\bu)\,.
\end{equation}
In what follows we use the fact that for all $x\leq 1$ we have:
\begin{equation}\label{xeq}
x(1-\exp(x))\geq -2 x^2\,.
\end{equation}
 For all $i\in[E]$\,, we have, by definition, that $\gc{t}{i}=\ext{i}{t}\cdot\hrv{t}$ so by Equation \eqref{algseceq2} we have:
 \begin{equation*}
 \gv{t}\cdot(\snw{t}-\bs{w}_{t+1})=\sum_{i\in[E]}\snc{t}{i}\ext{i}{t}\cdot\hrv{t}(1-\exp(\eta \ext{i}{t}\cdot\hrv{t}))\,.
 \end{equation*}
Since, for all $a\in[K]$\,, we have $\hrc{t}{a}\leq 1$ and hence, as $\eta<1$ and, for all $i\in[E]$ we have $\onorm{\ext{i}{t}}\leq 1$\,, we can invoke Equation \eqref{xeq}, which gives us:
\begin{equation}\label{geq1}
    \eta\gv{t}\cdot(\snw{t}-\bs{w}_{t+1})\geq-2\sum_{i\in[E]}\snc{t}{i}(\eta\ext{i}{t}\cdot\hrv{t})^2\,.
\end{equation}
By definition of $\hrv{t}$ we have, for all $i\in[E]$\,, that:
\begin{equation*}
    \ext{i}{t}\cdot\hrv{t}=\onorm{\ext{i}{t}}+e^i_{t,a_t}(1-r_{t,a_t})/\satc{t}{a_t}\\
    \leq\cnf{t}{i}+2e^i_{t,a_t}/\satc{t}{a_t}
\end{equation*}
so that since, for all $a\in[K]$\,, we have $\Pr[a_t=a]=\satc{t}{a}$ we also have:
\begin{equation}\label{geq2}
    \expt{(\ext{i}{t}\cdot\hrv{t})^2}\leq\cnf{t}{i}^2+\sum_{a\in[K]}(2e^i_{t,a}\cnf{t}{i}+4(e^i_{t,a})^2/\satc{t}{a})\,.
\end{equation}
Since, for all $i\in[E]$\, and $a\in[K]$, we have $e^i_{t,a}\leq1$ and $\cnf{t}{i}\leq 1$ and hence also $\cnf{t}{i}^2\leq\cnf{t}{i}$\, we then have:
\begin{equation}\label{geq4}
    \expt{(\ext{i}{t}\cdot\hrv{t})^2}\leq(2K+1)\cnf{t}{i}+4\sum_{a\in[K]}e^i_{t,a}/s_{t,a}\,.
\end{equation}
Note that since $\snw{t}\in\vat{t}$ we have:
\begin{equation}\label{geq5}
    \sum_{i\in[E]}\snc{t}{i}\cnf{t}{i}\leq1\,.
\end{equation}
Also, by definition of $\sat{t}$ we have:
\begin{align}
\label{geq6}
\sum_{i\in[E]}\snc{t}{i}\sum_{a\in[K]}e^i_{t,a}/\satc{t}{a}&=\sum_{a\in[K]}\frac{1}{\satc{t}{a}}\sum_{i\in[E]}\snc{t}{i}e^i_{t,a}
=\sum_{a\in[K]}\frac{1}{\satc{t}{a}}\satc{t}{a}=K\,.
\end{align}
Multiplying Inequality \eqref{geq4} by $\snc{t}{i}$\,, summing over all $i\in[E]$\,, and then substituting in Inequality \eqref{geq5} and Equation \eqref{geq6} gives us:
\begin{equation}\label{poeq1}
    \sum_{i\in[E]}\snc{t}{i}\expt{(\ext{i}{t}\cdot\hrv{t})^2}\leq(2K+1)+4K=6K+1\,.
\end{equation}
Taking expectations on Inequality \eqref{geq1} and substituting in  Inequality \eqref{poeq1} (after taking expectations) gives us:
\begin{equation} \label{poeq2}\expt{\eta\gv{t}\cdot(\snw{t}-\bs{w}_{t+1})}\geq-\eta^2(12K+2)\,.
\end{equation}
Taking expectations (over the draw $a_t\draw\sat{t}$) on Inequality \eqref{preq1}, substituting in Inequalities \eqref{eopeq2}, \eqref{preq2} and \eqref{poeq2}, and then rearranging gives us:
\begin{align*}
    \Delta(\bs{u},\bs{w}_t)-\expt{\Delta(\bs{u},\bs{w}_{t+1})}\geq\eta(\rf(\bu)-\rf(\snw{t}))-\eta^2(12K+2)\,.
\end{align*}
Summing this inequality over all $t\in[T]$\,, taking expectations (over the entire sequence of action draws) and noting that $\Delta(\bs{u},\bs{w}_{T+1})>0$ gives us:
\begin{equation*}
\Delta(\bs{u},\bs{w}_1)\geq\eta\sum_{t\in[T]}\expt{\rf(\bu)-\rf(\snw{t})}-\eta^2(12K+2)T\,.
\end{equation*}
Substituting in Equation \eqref{algseceq1} and rearranging then gives us, by definition of $\trf$ and $\rf$, the required goal:
\begin{equation*}
\sum_{t\in[T]}\expt{r_{t,a_t}}\geq\expt{\trf(\bu)}-\Delta(\bs{u},\bs{w}_1)/\eta-\eta(12K+2)T\,.
\end{equation*}    
\end{proof}

\subsection{Unbounded precision case} \label{apx:precision}

We will now show how to handle the case in which our numerical precision is unbounded, incurring a time complexity equal, up to a factor logarithmic in $T$, to that of Exp4 and adding only 1 to the regret. This additive factor, however, can be made arbitrarily small.

Let us restrict ourselves to compare against $\boldsymbol{u}$ with $\|\bs{u}\|_{\infty}\leq Z$ for some arbitrary $Z$. Note that this always has to be the case when each expert has a confidence of at least $1/Z$ on some trial. Our time complexity will be logarithmic in $Z$. At the beginning of trial $t$ we will now project (via the unnormalised relative entropy) $\boldsymbol{w}_t$ into the set $\{ \boldsymbol{v}\in\mathbb{R}^E \mid \|\bs{v}\|_{\infty} \leq Z \}$ which simply requires clipping its components. Since the set $\{ \boldsymbol{v}\in\mathbb{R}^E \mid \|\bs{v}\|_{\infty} \leq Z \}$ is convex and contains our comparator $\boldsymbol{u}$ this will not affect our regret bound.

For any $q\in\mathbb{R}$ let $\mathcal{V}_t(q)$ be the set of all $\boldsymbol{v}$ with $\boldsymbol{v}\cdot\boldsymbol{c}_t\leq q$. We note that given, for all $t\in[T]$, a value $q_t\in[1-1/T,1]$ we have that there exists $\hat{\boldsymbol{u}}\in\bigcap_t\mathcal{V}_t(q_t)$ such that the cumulative reward of $\boldsymbol{\pi}(\hat{\boldsymbol{u}})$ is no less than that of $\boldsymbol{\pi}(\boldsymbol{u})$ minus $1$. This means that, on any trial $t$ we can, instead of projecting into the set $\vat{t}$\,, project into the set $\mathcal{V}_t(q_t)$ for some $q_t\in[1-1/T,1]$ and add no more than one to the regret (by considering $\hat{\boldsymbol{u}}$ as the comparator instead of $\boldsymbol{u}$).

So the problem (for the projection step at time $t$ if necessary) is now to project into the set of all {$\{\boldsymbol{v}\,|\,\boldsymbol{v}\cdot\boldsymbol{c}\leq q_t\}$} for some arbitrary $q_t\in[1-1/T,1]$. Following our use of Lagrange multipliers, this means that we need to find $\lambda>0$ with $\sum_{i}c_{t,i}w_{t,i}\exp(-\lambda c_{t,i})\in[1-1/T,1]$. So consider the function $f$ defined by $f(\lambda'):=\sum_{i}c_{t,i}w_{t,i}\exp(-\lambda' c_{t,i})$.

Consider $\lambda':=ZE\ln(ZE)$ and take any $i\in[E]$. Since $w_{t,i}\leq Z$ we have that when $c_{t,i}<1/ZE$ then $c_{t,i}w_{t,i}\exp(-\lambda' c_{t,i})\leq c_{t,i}w_{t,i}<1/E$ and that when $c_{t,i}\geq 1/ZE$ then $c_{t,i}w_{t,i}\exp(-\lambda' c_{t,i})\leq Z\exp(-\lambda'/ZE)=1/E$. This implies that $f(\lambda')\leq 1$ and hence (since $f$ is monotonic decreasing) an acceptable $\lambda$ lies in $[0,ZE\ln(ZE)]$.

For general $\lambda'$ we note that $\nabla f(\lambda')=-\sum_i c_{t,i}^2w_{t,i}\exp(-\lambda' c_{t,i})\geq- f(\lambda')$. This means that $|\nabla f(\lambda)|\leq 1$. Since the length of the interval $[1-1/T,1]$ is $1/T$ this means that the length of the interval containing acceptable values of $\lambda$ is at least $1/T$.

So we have shown that either $\lambda=ZE\ln(ZE)$ is acceptable or the range of acceptable values of $\lambda$ is of length $1/T$ and lies in $[0,ZE\ln(ZE)]$ (which has length $ZE\ln(ZE)$). The ratio of these lengths is $ZET\ln(ZE)$ so interval bisection will find an acceptable value of $\lambda$ in $O(\ln(ZET\ln(ZE)))=O(\ln(EZT))$ steps.

So we have a time complexity $O(EK+E\ln(EZT))$ and we have only added $1$ to the regret (although this additive factor can be made arbitrarily small).

\section{Efficient implementation proof}\label{apx:efficient_imp}

We here prove the time complexity of \Cref{ballth}.
The per-trial time complexity of a direct implementation of \ose\ for this set of experts would be $\mathcal{O}(KN^2)$. We now show how to implement \ose\ in a per-trial time of only $\mathcal{O}(KN\ln(N))$. To do this first note that we can assume, without loss of generality, that for all $q,x,z\in\cons$ with $x\neq z$ we have $d(q,x)\neq d(q,z)$ since ties can be broken arbitrarily and balls can be duplicated. 

Given $x,z\in\cons$\,, $a\in[K]$ and $t\in[T]$ we let $y_{t,a}(x,z):=w_{t,i}$ and $\ty_{t,a}(x,z):=\snc{t}{i}$ where $i$ is the index of the expert corresponding to the ball-action pair with ball:
    $\{q\in\cons\,|\,d(x,q)\leq d(x,z)\},$
and action $a$.
Given $x,z\in\cons$ let
$\mf(x,z):=\{q\in\cons\,|\,d(x,q)\geq d(x,z)\}\,.$
It is straightforward to derive the following equations for the quantities in \ose\ at trial $t\in[T]$. First we have:
\begin{equation*} \onorm{\cnb{t}}=\sum_{a\in[K]}\sum_{x\in\cons}\sum_{z\in\mf(x,x_t)}y_{t,a}(x,z)\,.
\end{equation*}
For all $x,z\in\cons$ and $a\in[K]$ we have the following:
\begin{itemize}
\setlength{\parskip}{0pt}
\setlength{\itemsep}{0pt}
   \item If $\onorm{\cnb{t}}$$\leq$$1$ or $z\notin\mf(x,x_t)$ then $\ty_{t,a}(x,z)$$=$$y_{t,a}(x,z)$.
   \item If $\onorm{\cnb{t}}>1$ and $z\in\mf(x,x_t)$ then $\ty_{t,a}(x,z)=y_{t,a}(x,z)/\onorm{\cnb{t}}$.
\end{itemize}
For all $a\in[K]$ we have:
\begin{equation*} 
\satc{t}{a}=\sum_{x\in\cons}\sum_{z\in\mf(x,x_t)}\ty_{t,a}(x,z)\,.
\end{equation*}
Finally, for all $x,z\in\cons$ and $a\in[K]$ we have the following:
\begin{equation*}
y_{(t+1),a}(x,z) = \begin{cases*}
                    \ty_{t,a}(x,z) & if  $z\notin\mf(x,x_t)$\,,  \\
                     \ty_{t,a}(x,z)\exp(\eta \ext{i}{t}\cdot\hrv{t}) & if $z\in\mf(x,x_t)$\,.
                 \end{cases*} 
\end{equation*}

\begin{algorithm}[t]
\caption{\textsc{Query}$(\q)$}
\label{alg:Query}
\begin{enumerate}
\setlength{\parskip}{0pt}
\setlength{\itemsep}{0pt}
    \item For all $i\in[n]\cup\{0\}$ let $\gamma_{i}$ be the ancestor of $\q$ at depth $i$ in $\mathcal{D}$
    \item Set
$\sigma_n\la\psi(\gamma_{n})\phi(\gamma_{n})$
\item Climb $\mathcal{D}$ from $\gamma_{n-1}$ to $\gamma_{0}$. When at $\gamma_{i}$  do as follows:
\begin{enumerate}
\setlength{\parskip}{0pt}
\setlength{\itemsep}{0pt}
\item If $\gamma_{i+1}=\lc{\gamma_i}$ then set $\sigma_i\la\phi(\gamma_i)(\sigma_{i+1}+\psi(\rc{\gamma_i})\phi(\rc{\gamma_i}))$
\item If $\gamma_{i+1}=\rc{\gamma_i}$ then set $\sigma_i\la\phi(\gamma_i)\sigma_{i+1}$
\end{enumerate}
\item Return $\sigma_0$
\end{enumerate}
\end{algorithm}

\begin{algorithm}[t]
\caption{\textsc{Update}$(\q,\uc)$}
\label{alg:Update}
\begin{enumerate}
\setlength{\parskip}{0pt}
\setlength{\itemsep}{0pt}
    \item For all $i\in[n]\cup\{0\}$ let $\gamma_{i}$ be the ancestor of $\q$ at depth $i$ in $\mathcal{D}$
    \item\label{ups2} Descend $\mathcal{D}$ from $\gamma_0$ to $\gamma_{n-1}$. When at $\gamma_i$ set:
\begin{enumerate}
\setlength{\parskip}{0pt}
\setlength{\itemsep}{0pt}
\item $\phi(\lc{\gamma_i})\la\phi(\gamma_i)\phi(\lc{\gamma_i})$ 
\item $\phi(\rc{\gamma_i})\la\phi(\gamma_i)\phi(\rc{\gamma_i})$
\item  $\phi(\gamma_i)\la1$
\end{enumerate}
\item\label{ups3} For all $i\in[n-1]\cup\{0\}$, if $\gamma_{i+1}=\lc{\gamma_{i}}$ then set $\phi(\rc{\gamma_{i}})\la \uc\phi(\rc{\gamma_{i}})$ 
\item\label{ups4} Set $ \phi(\gamma_n)\la \uc\phi(\gamma_n)$ 
\item\label{ups5} Climb $\mathcal{D}$ from $\gamma_{n-1}$ to $\gamma_0$. When at $\gamma_i$ set:\\
   $\psi(\gamma_i)\la\psi(\lc{\gamma_i})\phi(\lc{\gamma_i})+\psi(\rc{\gamma_i})\phi(\rc{\gamma_i})$
\end{enumerate}
\end{algorithm}


Hence, to implement \ose\ we need, for each $x\in\cons$ and $a\in[K]$\,, a data structure that implicitly maintains a function $\hy:\cons\rightarrow\mathbb{R}^+$ and has the following two subroutines, that take parameters $\q\in\cons$ and $\uc\in\mathbb{R}_+$.
\begin{enumerate}
\setlength{\parskip}{0pt}
\setlength{\itemsep}{1pt}
\item \textsc{Query}$(\q)$:~ Compute $\sum_{z\in\mf(x,\q)}\hy(z)$.
\item \textsc{Update}$(\q,\uc)$:~ Set $\hy(z)\la \uc \hy(z)$ for all $z\in\mf(x,\q)$.
\end{enumerate}
Now fix $x\in\cons$ and $a\in[K]$. Let $\hy$ be as above. On each trial $t\in[T]$ and for all $z\in\cons$\,, $\hy(z)$ will start equal to $y_{t,a}(x,z)$ and change to $\ty_{t,a}(x,z)$ and then $y_{(t+1),a}(x,z)$ by applying the \textsc{Update} subroutine.

We now show how to implement these subroutines implicitly in a time of $\mathcal{O}(\ln(N))$ as required. Without loss of generality, assume that $N=2^n$ for some $n\in\mathbb{N}$. Our data structure is based on a balanced binary tree $\mathcal{D}$ whose leaves are the elements of $\cons$ in order of increasing distance from $x$. This implies that for any $z\in\cons$ we have that $\mf(x,z)$ is the set of leaves that do not lie on the left of $z$. Given a node $v\in\mathcal{D}$ we let $\anc{v}$ be the set of ancestors of $v$ and let $\des{v}$ be the set of all $z\in\cons$ which are descendants of $v$. For any internal node $v$ let $\lc{v}$ and $\rc{v}$ be the left and right children of $v$ respectively.

We maintain functions $\phi,\psi:\mathcal{D}\rightarrow\rplus$ \,such that for all $v\in\mathcal{D}$ we have:
\begin{equation}\label{dseq1}
\psi(v)\prod_{v'\in\anc{v}}\phi(v')=\sum_{z\in\des{v}}\hy(z)\,.
\end{equation}

The pseudo-code for the subroutines \textsc{Query} and \textsc{Update} are given in Algorithms \ref{alg:Query} and \ref{alg:Update} respectively. We now prove their correctness.
We first consider the \textsc{Query} subroutine with parameter $\q\in\cons$. From Equation \eqref{dseq1} we see that, by (reverse) induction on $i\in[n]\cup\{0\}$, we have:
\begin{equation*}
\sigma_i\prod_{v'\in\anc{\gamma_i}\setminus\{\gamma_i\}}\phi(v')=\sum_{z\in\des{\gamma_i}\cap\mf(x,\q)}\hy(z)\,.
\end{equation*}
Since $\gamma_0$ is the root of $\mathcal{D}$, we have $\sigma_0=\sum_{z\in\mf(x,\q)}\hy(z)$ as required.
Now consider the \textsc{Update} subroutine with parameters $q\in\cons$ and $\uc\in\mathbb{R}_+$. Let $\hy$ be the implicitly maintained function before the subroutine is called. For Equation \eqref{dseq1} to hold after the subroutine is called we need:
\begin{equation}\label{dseq2}
\psi(v)\prod_{v'\in\anc{v}}\phi(v')=\sum_{z\in\des{v}}\hy'(z)\,.
\end{equation}
where for all $z\in\cons$ we have:
\begin{equation*}
\hy'(z):=\indi{z\notin\mf(x,\q)}\hy(z)+\indi{z\in\mf(x,\q)}\uc\hy(z)\,.
\end{equation*}
We shall now show that Equation \eqref{dseq2} does indeed hold after the subroutine is called, which will complete the proof. To show this we consider each step of the subroutine in turn. After Step \ref{ups2} we have (via induction) that:
\begin{itemize}
\setlength{\parskip}{0pt}
\setlength{\itemsep}{0pt}
\item For all $v\in\anc{\q}$ we have $\phi(v)=1$.
\item For all $v\in\mathcal{D}\setminus\anc{\q}$ we have:
\begin{equation*}
\psi(v)\prod_{v'\in\anc{v}}\phi(v')=\sum_{z\in\des{v}}\hy(z)\,.
\end{equation*}
\end{itemize}
So, since $\mf(x,\q)$ is the set of all $z\in\cons$ that do not lie to the left of $\q$ in $\mathcal{D}$ we have that, after Step \ref{ups4} of the algorithm, the following holds:
\begin{itemize}
\setlength{\parskip}{0pt}
\setlength{\itemsep}{0pt}
\item For all $v\in\anc{\q}$ we have $\phi(v)=1$,
\item For all $v\in\mathcal{D}\setminus\anc{\q}$ we have:
\begin{equation*}
\psi(v)\prod_{v'\in\anc{v}}\phi(v')=\sum_{z\in\des{v}}\hy'(z)\,.
\end{equation*}
\end{itemize}
Hence, by induction, we have that, after Step \ref{ups5} of the algorithm, it is the case that for all $v\in\anc{\q}$ we have:
$\psi(v)=\sum_{z\in\des{v}}\hy'(z)\,$. 
So since $\phi(v)=1$ for all $v\in\anc{\q}$ and Step~\ref{ups5}
does not alter $\phi(v)$ or $\psi(v)$ for any $v\in\mathcal{D}\setminus\anc{\q}$ we have Equation \eqref{dseq2}.
\hfill$\blacksquare$

\section{Lower bound proof}\label{apx:lowerbound}
\begin{proposition}
Take any learning algorithm. Given any basis $\basis$ and any $M\in\mathbb{N}$ then for any sequence of disjoint basis elements $\seq{\bll_j}{j\in[M]}$ there exists a sequence of corresponding actions $\seq{\cra{j}\in[K]}{j\in[M]}$ such that an adversary can force:
\begin{equation*}
    \sum_{t\in [T]}\sum_{j\in[M]} \indi{x_t \in \mathcal{B}_j}r_{t,b_j} -
    \sum_{t\in [T]} \mathbb{E}[r_{t,a_t}]\in \Omega(\sqrt{MKT})
\end{equation*}    
\end{proposition}
\begin{proof}
In this scenario, at each time step, either a single expert (i.e., the basis element containing the current context $x_t$) is active, making predictions based on its label, or no expert is active, prompting the learner to abstain and thus incur zero reward or cost.

Therefore we define $T' = \{t \in [T] \,|\, \sum_{j\in[M]} \indi{x_t \in \mathcal{B}_j} = 1\}$ as the set of timesteps in which the learner is going to play. Since the concept of abstention is that our algorithm is not going to pay anything for the timesteps in which we abstain, we can see that:

\begin{align*}
    \sum_{t\in [T]}\sum_{j\in[M]} \indi{x_t \in \mathcal{B}_j}r_{t,b_j} -
    \sum_{t\in [T]} \mathbb{E}[r_{t,a_t}]
    =
    \sum_{t\in T'}r_{t,b_j} -
    \sum_{t\in T'} \mathbb{E}[r_{t,a_t}] \,,
\end{align*}

For any ball $j \in [M]$, we define $T_j = \{t \in [T'] \,|\, \indi{x_t \in \mathcal{B}_j}\}$. Following the ideas of \citet{seldin2016lower}, for any of the sets $T_j$ we can create a multi-armed bandit instance as the one described in the lower bound by \citet{auer2002nonstochastic}.
Note that in the lower bound construction, the abstention arm would be a forehand known suboptimal arm, which results in a lower bound of the order $c\sqrt{(K-1)T}$, for the constant $c = \frac{\sqrt{2} - 1}{\sqrt{32\ln(4/3)}} > 0$.
Since the presented context $x_t$ is chosen adversarially at each time step, we can ensure that each basis element is activated for $|T'|/M$ time steps, obtaining: 

\begin{align*}
    \sum_{j \in [M]} \left( \sum_{s \in T'_j} r_{s,b_j} -
    \sum_{s\in T'_j} \mathbb{E}[r_{s,a_s}] \right)
    & \geq
    \sum_{j \in [M]} c\sqrt{(K-1)|T'_j|} 
    \\ & =
    \sum_{j \in [M]} c\sqrt{(K-1)|T'|/M} 
    \\ & =
    c\sqrt{M(K - 1)|T'|}
\end{align*}

As we can choose $|T'|$ to be any fraction of T, we end up with the desired lower bound of the order $\Omega(\sqrt{MKT})$, which matches, up to logarithmic factors, the cumulative reward bound presented in \Cref{ballth}.
\end{proof}

\section{Overlapping balls extension}
\label{apx:overlapping}
In this section, we present the theorem that allows us to present the results of overlapping balls as expressed in Section \ref{sec:efficientlearning}. Note that Theorem \ref{ballth} is the special case of Theorem \ref{thm:OBE4overlapping} when the balls are disjoint and $\uw{j}=1$ for all $j\in[M]$.

\begin{theorem}
\label{thm:OBE4overlapping}
Let $M\in\mathbb{N}$ and $\{(\bll_j,\cra{j},\uw{j})\,|\,j\in[M]\}$ be any sequence 
such that $\bll_j$ is a ball, $\cra{j}\in[K]$ is an action, and $\uw{j}\in[0,1]$ is such that for all $x\in\cons$ we have:
\begin{equation*}
 \sum_{j\in[M]}\indi{x\in \bll_j}\uw{j}\leq 1\,.
\end{equation*}
For all $t\in[T]$ define:
\begin{equation*}
{r}^*_t:=\sum_{j\in[M]}\indi{x_t\in \bll_j}\uw{j}{r}_{t,\cra{j}}\,,
\end{equation*}
which represents the reward of the policy induced by 
$\{(\bll_j,\cra{j},u_j)\,|\,j\in[M]\}$ on trial $t$.
The regret of \ose, with the set of experts given in Section \ref{sec:efficientlearning} and with correctly tuned parameters, is then bounded by:
\begin{equation*}
\sum_{t\in[T]}{r}^*_t-\sum_{t\in[T]}\mathbb{E}[{r}_{t,a_t}]\in \mathcal{O}\left(\sqrt{\ln(KN)KT\sum_{j\in[M]}u_j}\right)\,.
\end{equation*}
Its per-trial time complexity is: 
\begin{equation*}
\mathcal{O}(KN\ln(N))\,.
\end{equation*}
\end{theorem}

\begin{proof}
Direct from Theorem \ref{cbath} using the experts (with efficient implementation) given in Section \ref{sec:efficientlearning}
\end{proof}

\section{The details of the graph bases}
\label{sec:thedeialsofthebases}


This section expands the definition and explanations for the bases we used in the Experiment. Remember that we refer to any set of experts that correspond to set-action pairs of the form $(\bll,k)\in 2^\cons\times[K]$ as a \emph{basis elements}, and a set of basis elements as \emph{basis}.


\subsection{$p$-seminorm balls on graphs}

\SSS{Do we explicitly mention like if we choose specialists like GABA we strictly improve the regret bound of Gaba?}\MT{this would be nice, even just recovering Gaba's bound. But I thought we were not sure how to do that.}

As we see in Sec.~\ref{sec:efficientlearning}, the \ose\ seems to work only for vector data. 
However, in the following sections, we explore how our \ose\ algorithm can be applied to graph data by creating a ball structure over the graph.

We first introduce the notations of a graph. A graph is a pair of \emph{nodes} $V := [N]$ and \emph{edges} $E$. An edge connects two nodes, and we assume that our graph is \emph{undirected} and \emph{weighted}. 
For each edge $\{i,j\} \in E$, we denote its weight by $c_{ij}$. 
For convenience, for each pair of nodes $i,j$ with $\{i,j\}\notin E$, we define $c_{ij}=0$. 

To form a ball over a graph, a family of metrics we are particularly interested in is given by $p$-norms on a given graph $G$. 
Let
\begin{equation}
\label{eq:pmet}
d_p(i,j) := \left({\min\limits_{\substack{\bs{u}\in\R^N\\u_{i}-u_{j}=1}} \sum\limits_{s,t\in V} c_{st}|u_s - u_t|^p}\right)^{-1/p}\,.
\end{equation}
which is a well-defined metric for $p\in [1,\infty)$ if the graph is connected and may be defined for $p=\infty$ by taking the appropriate limits.  When $p=2$ this is the square root of the {\em effective resistance} circuit between nodes $i$ and $j$ which comes from interpreting the graph as an electric circuit where the edges are unit resistors and the denominator of Equation~\eqref{eq:pmet} is the power required to maintain a unit voltage difference between $u$ and $v$~\citep{doyle1984random}.  
More generally, $d_p(i,j)^p$ is known as $p$-(effective) resistance~\citep{herbster2009predicting,alamgir2011phase,saito2023multi}.    
When $p\in \{1,2,\infty\}$ there are natural interpretation of the $p$-resistance.  
In the case of $p=1$, we have that the effective is equal to one over the number of edge-disjoint paths between $i$ and $j$ which is equivalently one over the minimal cut that separates $i$ from $j$.  
When $p=2$ it is the effective resistance as discussed above.  
And finally when $p=\infty$ we have that $d_{\infty}$ is the geodesic distance (shortest path) between $i$ and $j$.
Note that, interestingly, there are at most $2N$ distinct balls for $d_1$; as opposed to the general bound $O(N^2)$ on the number of metric balls. 
This follows since $d_1$ is an \emph{ultrametric}.
A nice feature of metric balls is that they are ordinal, i.e., we can take an increasing function of the distance and the distinct are unchanged.
The time complexity for each ball is as follows.
For $d_{1}$ ball, we compute every pair of distance in $\mathcal{O}(N^3)$ using the Gomory-Hu tree~\citep{gomory1961multi}.
For $d_{2}$ ball, it is actually enough to compute the pseudoinverse of graph Laplacian once, which costs $\mathcal{O}(N^3)$~\citep{doyle1984random}.
For $d_{\infty}$ ball, we can compute every pair of distance in $\mathcal{O}(N^3)$ by Floyd–Warshall algorithm~\citep{floyd1962algorithm}.





\subsection{Community detection bases}

In this section, we consider only bases formed via a set of subsets (a.k.a clusters) $C \subseteq 2^{[N]}$. Each of these subsets induces $K$ basis elements: one for each action $a\in[K]$. Specifically, the basis element $\beta:[N]\rightarrow\Ksq$ corresponding to the pair $(C,a)$ is such that $\beta(x)$ is equal to $a$ whenever $x\in C$ and equal to $\square$ otherwise. Hence, in this section, we equate a basis with a set of subsets of $[N]$.

We can compute a basis for a given graph $G=(V,E)$ using community detection algorithms. 
Community detection is one of the most well-studied operations for graphs, where the goal is to find a partition $\{C_1,\dots, C_q\}$ of $V$ (i.e., $\bigcup_{i=1}^q C_i=V$ and $C_i\cap C_j=\emptyset$ for $i\neq j$) so that each $C_i$ is densely connected internally but sparsely connected to the rest of the graph~\citep{Fortunato2010community}. 
There are many community detection algorithms, all of which can be used here, but the most popular algorithm is the Louvain method~\citep{blondel2008fast}. 
We briefly describe how this algorithm works. The algorithm starts with an initial partition $\{\{v\}\mid v\in V\}$ and aggregates the clusters iteratively: 
For each $v \in V$, compute the gain when moving $v$ from its current cluster to its neighbors' clusters and indeed move it to a cluster with the maximum gain (if the gain is positive). 
Note that the gain is evaluated using \emph{modularity}, i.e., the most popular quality function for community detection~\citep{newman2004finding}. 
The algorithm repeats this process until no movement is possible. 
Then the algorithm aggregates each cluster to a single super node (with appropriate addition of self-loops and change of edge weights) 
and repeats the above process on the coarse graph as long as the coarse graph is updated. 
Finally, the algorithm outputs the partition of $V$ in which each cluster corresponds to each super node in the latest coarse graph. 
Note that it is widely recognized that the Louvain method works in $\mathcal{O}(N \log N)$ in practice~\citep{traag2015faster}.
\SSS{Feel free to fix this sentence to whatever you like @atsushi}

To obtain a finer-grained basis, we apply the so-called greedy peeling algorithm for each $C_i$ in the output of the Louvain method. 
For $C_i\subseteq V$ and $v\in C_i$, we denote by $d_{C_i}(v)$ the degree of $v$ in the induced subgraph $G[C_i]$. 
For $G[C_i]$, the greedy peeling iteratively removes a node with the smallest degree in the currently remaining graph and obtains a sequence of node subsets from $C_i$ to a singleton. Specifically, it works as follows: Set $j\leftarrow |C_i|$ and $C_i^{(j)}\leftarrow C_i$. For each $j=|C_i|,\dots, 2$, compute $v_\text{min}\in \argmin\{d_{C_i^{(j)}}(v)\mid v\in C_i^{(j)}\}$ and $C_i^{(j-1)}\leftarrow C_i^{(j)}\setminus \{v_\text{min}\}$. 
Using a sophisticated data structure, this algorithm runs in linear time~\citep{lanciano2023survey}.

In summary, our community detection basis is the collection of node subsets $\{C_i^{(j)}\mid i = 1,\dots, q,\, j = 1,\dots, |C_i|\}$ together with $\{\{v\}\mid v\in V\}$ for completeness. 

\subsection{Graph convexity bases}
An alternative to metric balls and communities are, for example, (geodesically) convex sets in a graph. They correspond to the inductive bias that if two nodes prefer the same action, then also the nodes on a shortest path between the two tend to prefer the same action. Geodesically convex sets are well-studied \citep{van1993theory,pelayo2013geodesic} and have been recently used in various learning settings on graphs \citep{bressan2021exact, thiessen2021active}. 
Similarly to convex sets in the Euclidean space, a set $C$ of nodes is \emph{convex} if the nodes of any shortest path with endpoints in $C$ are in $C$, as well. More formally, the (geodesic) \emph{interval} $I(u,v)=\{x\in V : x \text{ is on a shortest path between } u \text{ and } v\}$ of two nodes $u$ and $v$ contains all the nodes on a shortest path between them.
For a set of node $A$ we define $I(A)=\cup_{a,b\in A}I(a,b)$ as a shorthand notation for the union of all pairwise intervals in $A$. A set $A$ is (geodesically) convex iff $I(A)= A$ and the \emph{convex hull} $\conv(A)$ of a set $A$ is the (unique) smallest convex set containing $A$. Note that for $u,v\in V$, $I(u,v)$ and $\text{conv}(\{u,v\})$ are typically different sets. Indeed, $I(u,v)$ is in general non-convex, as nodes on a shortest path between two nodes in $I(u,v)$ (except for $u,v$) are not necessarily contained in $I(u,v)$.
%
As the total number of convex sets can be exponential in $N$, e.g., all subsets of a complete subgraph are convex, we consider the basis consisting of all intervals: $I(u,v)$ for $u,v\in[N]$. This involves $\mathcal{O}(N^2)$ basis elements, each of size $\mathcal{O}(N)$. With a simple modification of the Floyd Warshall \citep{floyd1962algorithm} algorithm, computing the interval basis takes $\mathcal{O}(N^3)$ time complexity.


\section{Additional experimental results} 
\label{apx:experiments}


We thoroughly explored various configurations for the three graphs described in our experimental setup in Section \ref{sec:experiments}. We run our experiments with an Intel Xeon Gold 6312U processor and 256 GB of RAM ECC 3200 MHz.
Figure \ref{fig:cliquetotal} displays different settings for the number of nodes in each clique and noise levels. 

As we compare the computational complexity of each basis in Section~\ref{sec:thedeialsofthebases} and the main results, the most intense computational load in the experiments will arise from the calculation of the basis, which can be seen as an initialization step in our algorithm. 
The proposed methods have varying computational complexities, and an arbitrarily complex function can be employed to compute the basis.
Remark that, in the usual complexity comparison among online learning algorithms using experts, we compare the complexity \textit{given} the experts.
Practically, we use pre-computed bases or even human experts.  
Also note that due to the expensive complexity of the $p$-balls and the convex sets seen in Section~\ref{sec:thedeialsofthebases}, we only conduct the LVC for LastFM Asia.

In Figure \ref{fig:gaussiantotal}, we present multiple settings for generating the Gaussian graph. Here the title of each plot is ``Foreground $x$,$y$; Background $x'$,$y'$; $k$-NN,'' which is explained as follows: $x$ represents the number of nodes in each foreground class, $x'$ represents the number of nodes in the background class, $y$ represents the standard deviation of the Gaussians generating the foreground class, $y'$ represents the standard deviation of the Gaussian generating the background class, and $k$ represents the number of nearest neighbors used to generate the graph.

In Figure \ref{fig:coratotal}, we present the various labels chosen as noise for the Cora graph. In Figure \ref{fig:coramain}, we presented the averages of all these different configurations. Here, we can see that the main behavior of the various bases is roughly maintained independently of the different labels chosen to be masked as background class. 

In Figure \ref{fig:lastfmtotal}, we present the various labels chosen as noise for the LastFM Asia graph. This graph comprises nodes representing LastFM users in Asian countries and edges representing mutual follower connections. Vertex features are extracted based on the artists liked by the users. During this initial analysis, we arbitrarily chose three out of eighteen possible labels to serve as the background class. In Figure \ref{fig:lastfmmain}, we presented the averages of all these different configurations. Varying the chosen background classes also produces different results, this is indeed due to the inherent lack of noise in the dataset. It is nice to see that regardless of the noise labels chosen, the behavior of our algorithm is always good, showing, as expected, that based on the amount of noise, we can just improve.

\clearpage

\begin{figure}[!t]
\centering
   \includegraphics[width=1.1\linewidth]{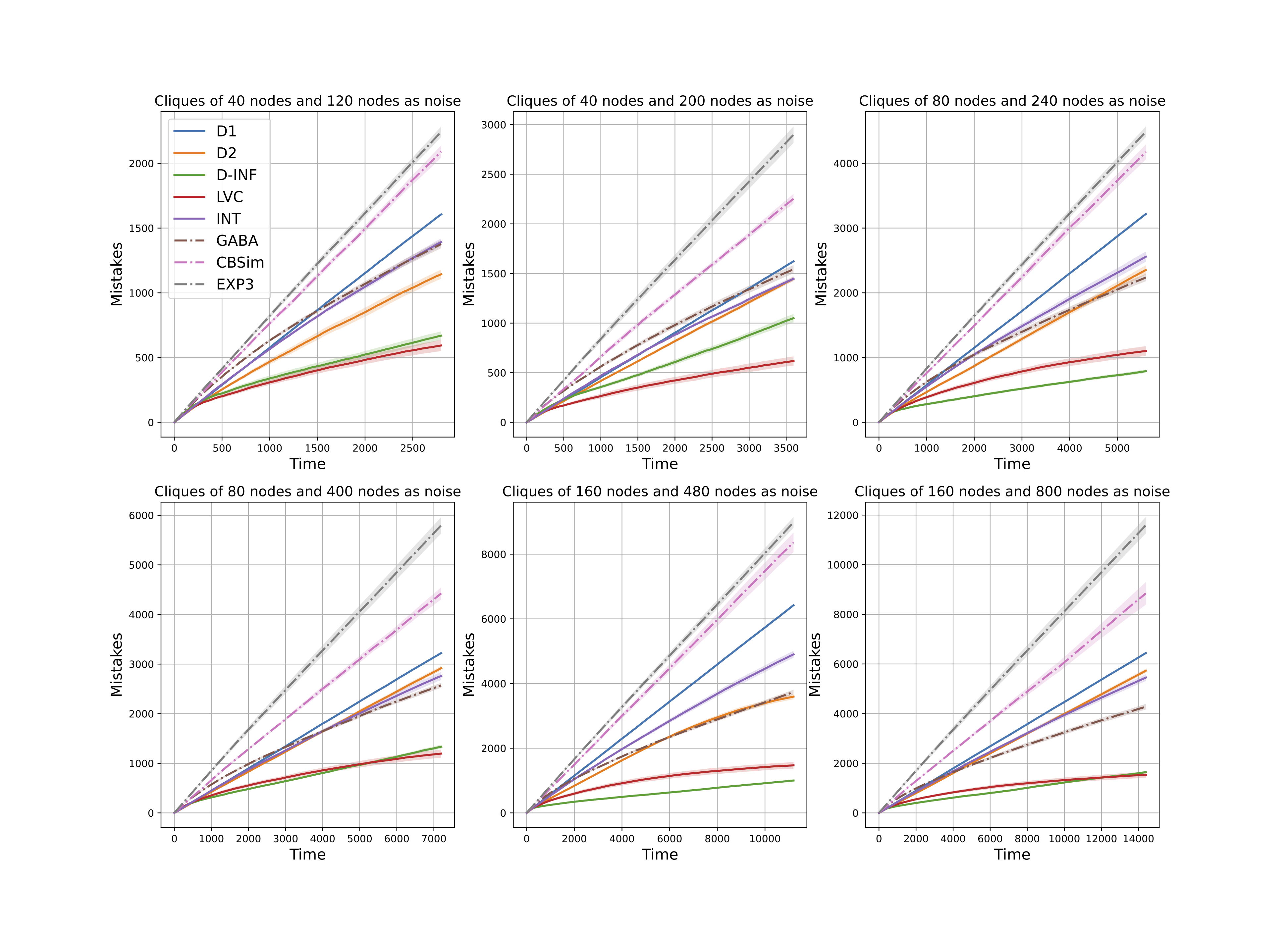}
   \caption{Stochastic Block Model results, dotted lines represent different baselines, while solid lines are used to represent various results.}
   \label{fig:cliquetotal}
\end{figure}

\clearpage

\begin{figure}[!t]
\centering
   \vspace{-2.1cm}
   \includegraphics[width=1.1\linewidth]{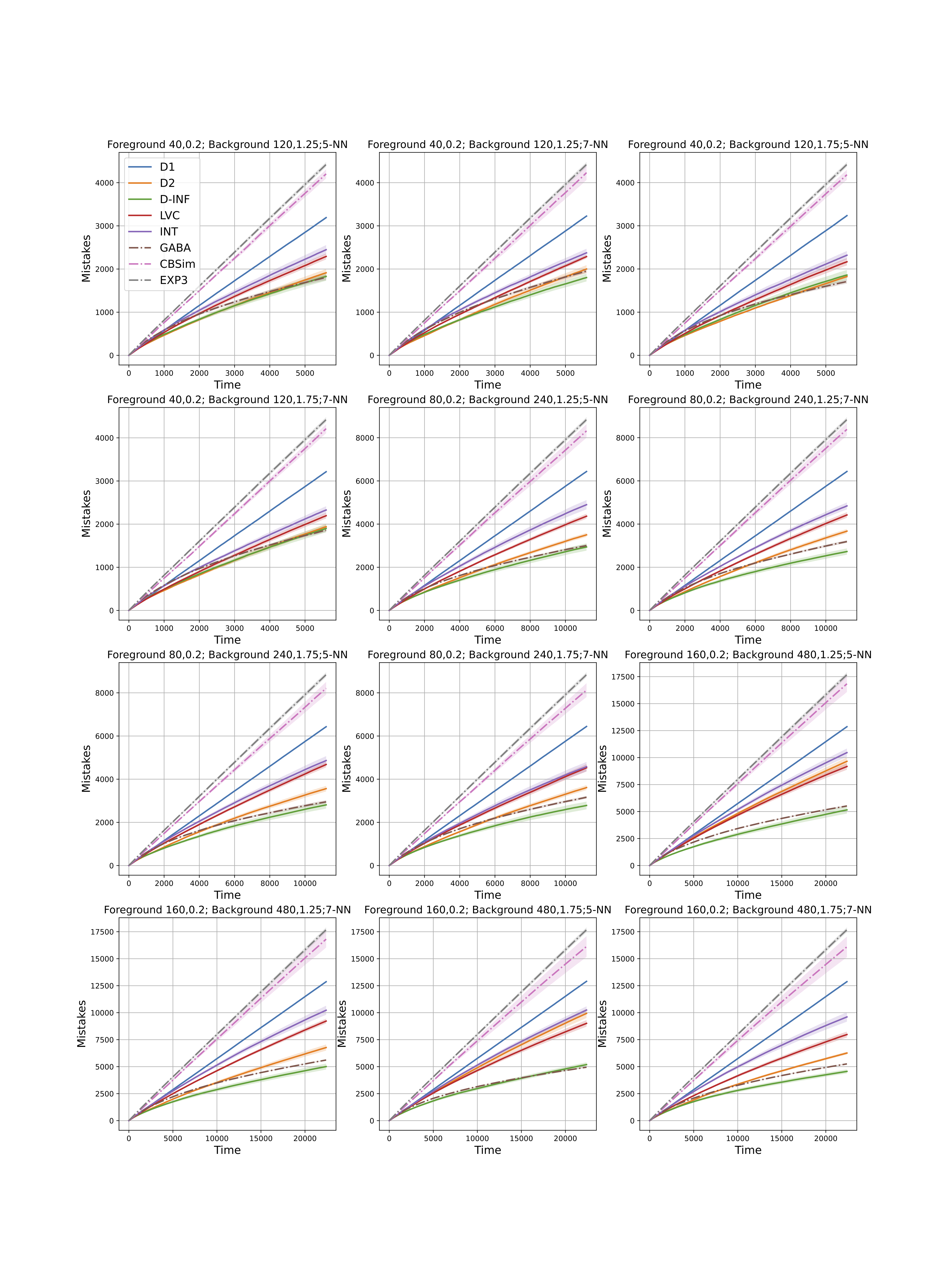}
   \vspace{-2.1cm}
   \caption{Gaussian graph results, dotted lines represent different baselines, while solid lines are used to represent various results.}
   \label{fig:gaussiantotal}
\end{figure}

\clearpage

\begin{figure}[!t]
\centering
   \includegraphics[width=1.1\linewidth]{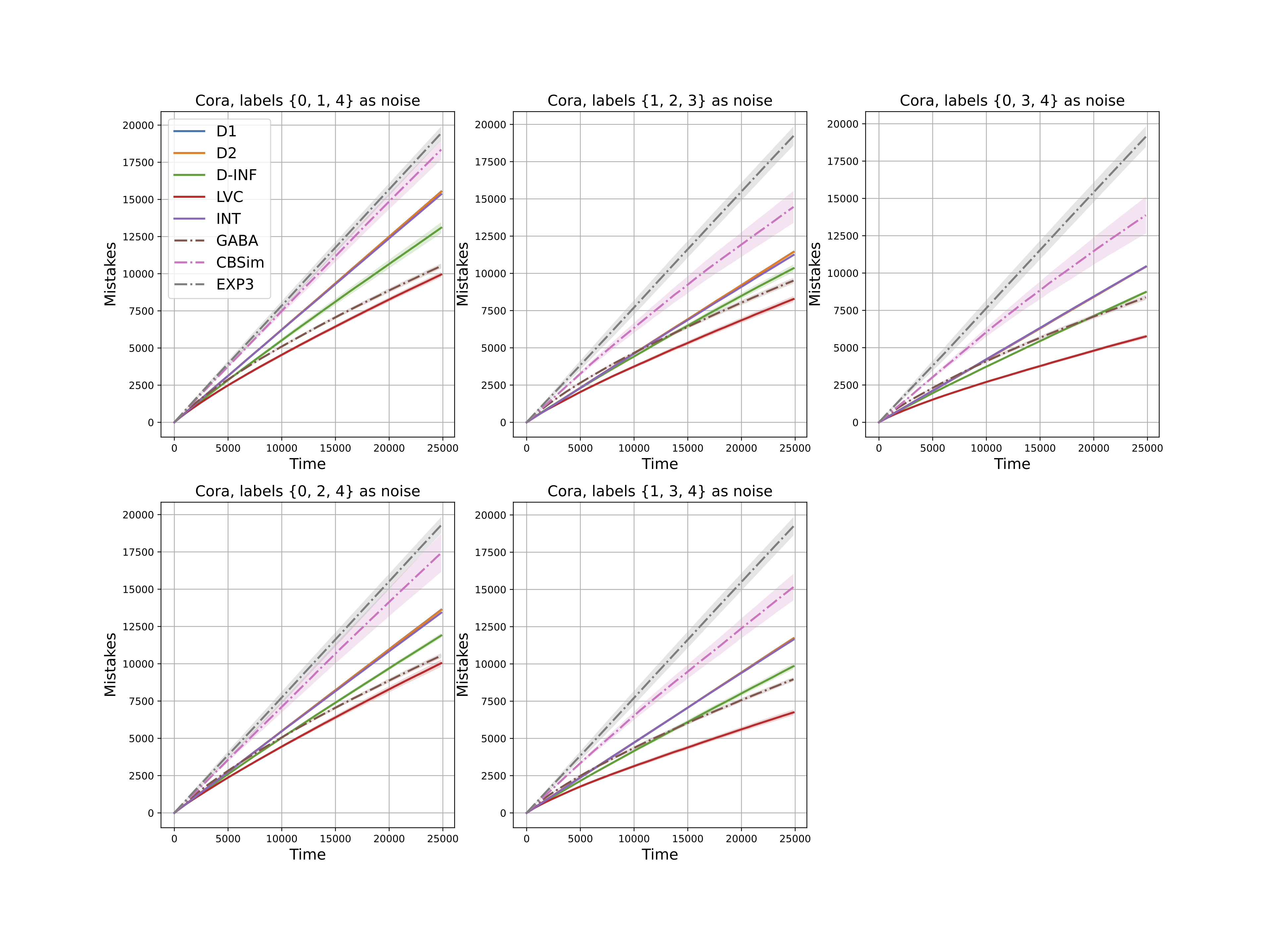}
   \caption{Cora results, dotted lines represent different baselines, while solid lines are used to represent various results}
   \label{fig:coratotal}
   \vspace{-2.1cm}
\end{figure}

\clearpage

\begin{figure}[!t]
\centering
   \includegraphics[width=1\linewidth]{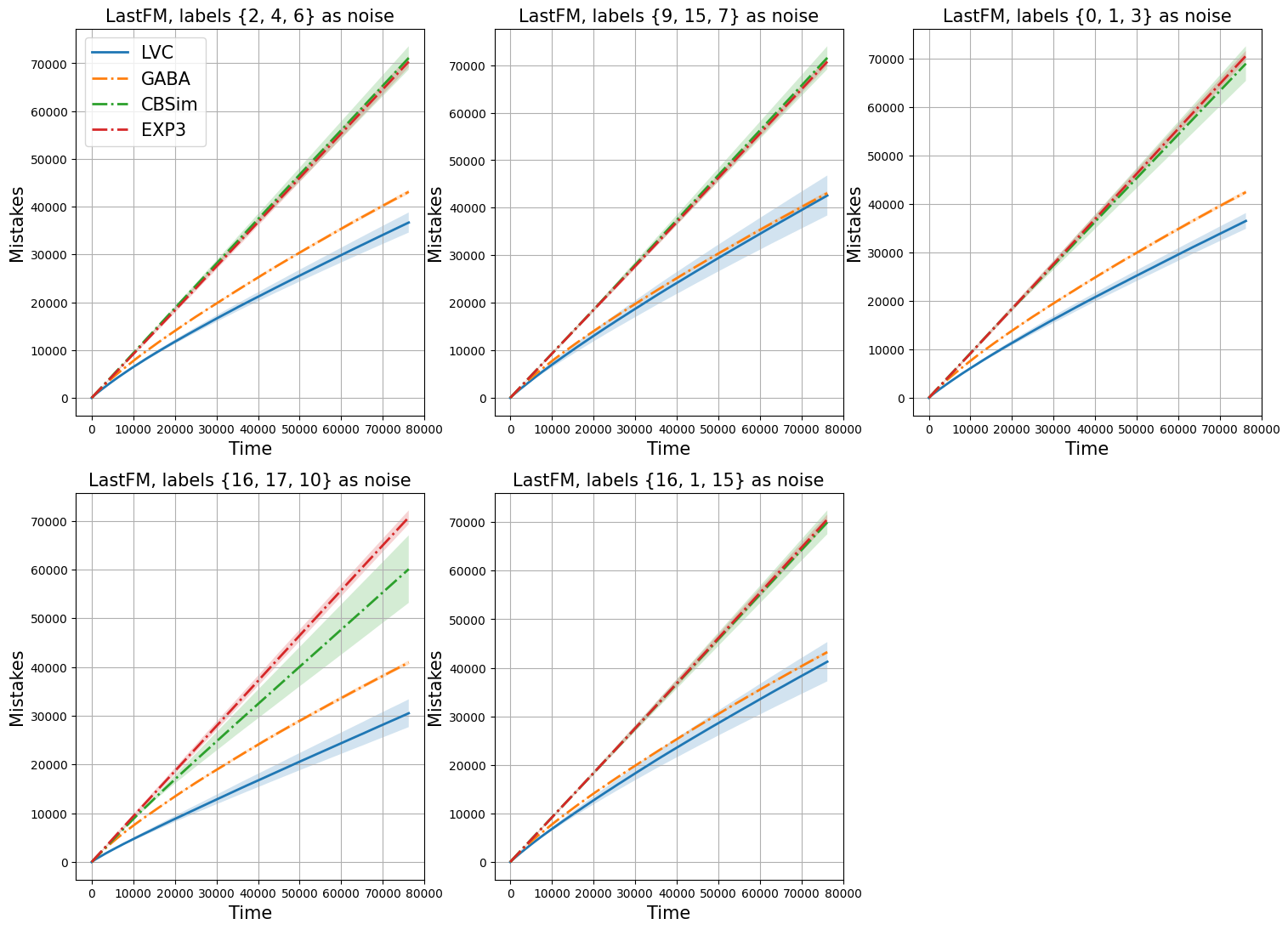}
   \caption{LastFM Asia results, dotted lines represent different baselines, while solid lines are used to represent various results}
   \label{fig:lastfmtotal}
\end{figure}


\clearpage

\section*{Impact Statement}
Given the theoretical nature of our work, we cannot foresee the shape of positive or negative societal impacts which this work may have in future.

\section*{NeurIPS Paper Checklist}
\begin{enumerate}

\item {\bf Claims}
    \item[] Question: Do the main claims made in the abstract and introduction accurately reflect the paper's contributions and scope?
    \item[] Answer: \answerYes{}
    \item[] Justification: All the claims are supported in the main body.

\item {\bf Limitations}
    \item[] Question: Does the paper discuss the limitations of the work performed by the authors?
    \item[] Answer: \answerYes{} 
    \item[] Justification: The limitations and future work are discussed in the introduction and in the experimental results analysis.

\item {\bf Theory Assumptions and Proofs}
    \item[] Question: For each theoretical result, does the paper provide the full set of assumptions and a complete (and correct) proof?
    \item[] Answer: \answerYes{} 
    \item[] Justification: We explicitly write the assumptions of all the theoretical claims.

    \item {\bf Experimental Result Reproducibility}
    \item[] Question: Does the paper fully disclose all the information needed to reproduce the main experimental results of the paper to the extent that it affects the main claims and/or conclusions of the paper (regardless of whether the code and data are provided or not)?
    \item[] Answer: \answerYes{} 
    \item[] Justification: We provided the experimental codes.

\item {\bf Open access to data and code}
    \item[] Question: Does the paper provide open access to the data and code, with sufficient instructions to faithfully reproduce the main experimental results, as described in supplemental material?
    \item[] Answer: \answerYes{} 
    \item[] Justification: We cited the datasets which we use in the experiments. Also, these datasets are publicly available and widely used in the community.

\item {\bf Experimental Setting/Details}
    \item[] Question: Does the paper specify all the training and test details (e.g., data splits, hyperparameters, how they were chosen, type of optimizer, etc.) necessary to understand the results?
    \item[] Answer: \answerYes{} 
    \item[] Justification: We provided the details in the main body as well as in the Appendix.

\item {\bf Experiment Statistical Significance}
    \item[] Question: Does the paper report error bars suitably and correctly defined or other appropriate information about the statistical significance of the experiments?
    \item[] Answer: \answerYes{} 
    \item[] Justification: We provided error bars and did statistical tests in the main body.

\item {\bf Experiments Compute Resources}
    \item[] Question: For each experiment, does the paper provide sufficient information on the computer resources (type of compute workers, memory, time of execution) needed to reproduce the experiments?
    \item[] Answer: \answerYes{} 
    \item[] Justification: We explicitly write the computing resources we used in the experiments in the Appendix.
    
\item {\bf Code Of Ethics}
    \item[] Question: Does the research conducted in the paper conform, in every respect, with the NeurIPS Code of Ethics \url{https://neurips.cc/public/EthicsGuidelines}?
    \item[] Answer: \answerYes{} 
    \item[] Justification: We have read and followed NeurIPS Code of Ethics.

\item {\bf Broader Impacts}
    \item[] Question: Does the paper discuss both potential positive societal impacts and negative societal impacts of the work performed?
    \item[] Answer: \answerYes{} 
    \item[] Justification: We provided the dedicated section for this.
    
\item {\bf Safeguards}
    \item[] Question: Does the paper describe safeguards that have been put in place for responsible release of data or models that have a high risk for misuse (e.g., pretrained language models, image generators, or scraped datasets)?
    \item[] Answer: \answerNA{} 
    \item[] Justification: The paper poses no such risks. 

\item {\bf Licenses for existing assets}
    \item[] Question: Are the creators or original owners of assets (e.g., code, data, models), used in the paper, properly credited and are the license and terms of use explicitly mentioned and properly respected?
    \item[] Answer: \answerYes{} 
    \item[] Justification: We used the datasets that are widely used in the community.

\item {\bf New Assets}
    \item[] Question: Are new assets introduced in the paper well documented and is the documentation provided alongside the assets?
    \item[] Answer: \answerNA{} 
    \item[] Justification: The paper does not release new assets. 

\item {\bf Crowdsourcing and Research with Human Subjects}
    \item[] Question: For crowdsourcing experiments and research with human subjects, does the paper include the full text of instructions given to participants and screenshots, if applicable, as well as details about compensation (if any)? 
    \item[] Answer: \answerNA{} 
    \item[] Justification: The paper does not involve crowdsourcing nor research with human subjects. 

\item {\bf Institutional Review Board (IRB) Approvals or Equivalent for Research with Human Subjects}
    \item[] Question: Does the paper describe potential risks incurred by study participants, whether such risks were disclosed to the subjects, and whether Institutional Review Board (IRB) approvals (or an equivalent approval/review based on the requirements of your country or institution) were obtained?
    \item[] Answer: \answerNA{} 
    \item[] Justification: The paper does not involve crowdsourcing nor research with human subjects.

\end{enumerate}
    

\end{document}